\documentclass[10pt,a4paper,twocolumn]{neatpreprint}

\usepackage[authoryear,round]{natbib}
\usepackage{mathtools,array}
\hypersetup{pdftitle={Rare Gate Disagreements Can Limit Plasticity: When Gradient Flow Mispredicts Finite-Batch SGD},
  pdfauthor={Ruoyu Zhao, Mingxuan Zhang, Jianbo Dai, Jiaqi Wu, Chenyu Zhu, Tong Che},
  linkcolor=blue!50!black,citecolor=blue!50!black,urlcolor=blue!50!black}
\usepackage[capitalise,nameinlink]{cleveref}
\newtheorem{theorem}{Theorem}
\newtheorem{proposition}[theorem]{Proposition}
\newtheorem{lemma}[theorem]{Lemma}
\newtheorem{corollary}[theorem]{Corollary}

\newcommand{\R}{\mathbb R}
\newcommand{\E}{\mathbb E}
\newcommand{\Pp}{\mathbb P}
\newcommand{\relu}{\sigma}
\newcommand{\calD}{\mathcal D}
\newcommand{\calS}{\mathcal S}
\newcommand{\norm}[1]{\left\lVert #1\right\rVert}
\newcommand{\abs}[1]{\left\lvert #1\right\rvert}
\newcommand{\ind}{\mathbf 1}
\newcommand{\budget}{\Gamma}

\crefname{appendix}{Appendix}{Appendices}
\crefname{subappendix}{Appendix}{Appendices}
\crefname{subsubappendix}{Appendix}{Appendices}

\title{Rare Gate Disagreements Can Limit Plasticity:\\ When Gradient Flow Mispredicts Finite-Batch SGD}
\runningtitle{Rare Gate Disagreements Can Limit Plasticity: When Gradient Flow Mispredicts Finite-Batch SGD}
\preprintlabel{Preprint}

\author[1]{Ruoyu Zhao}
\author[2]{Mingxuan Zhang}
\author[3]{Jianbo Dai}
\author[1]{Jiaqi Wu}
\author[1]{Chenyu Zhu}
\author[4,$\dagger$]{Tong Che}
\affil[1]{City University of Hong Kong, \texttt{\{ruoyuzhao8-c,jwu395-c,chenyuzhu9-c\}@my.cityu.edu.hk}}
\affil[2]{Microsoft, \texttt{mingxzhang@microsoft.com}}
\affil[3]{Copula Lab, \texttt{j6dj6d@gmail.com}}
\affil[4]{NVIDIA Research, \texttt{tongc@nvidia.com}}
\firstpagenote{$^\dagger$Project lead.}

\begin{document}

\begin{abstract}
Population gradient flow is a common tool for reasoning about how neural networks adapt, including after pretraining. We show that it can mispredict finite-batch stochastic gradient descent (SGD) qualitatively, and we trace the discrepancy to a specific mechanism. In a two-unit ReLU regression, a source task drives the two neurons toward positive proportionality and a target task rewards separating them. After source training for time $T$, gradient flow recovers on the target in time linear in $T$. Online SGD with batch size $b$ and step size $\eta$ in both phases instead fails with high probability throughout a horizon of order $e^{c/\eta}$ once $T\gtrsim\log(b/\eta)$, uniformly on an explicit set of initializations with Gaussian probability above one percent. For each fixed $T$, small-step SGD still recovers, so the failure requires the joint limit of small steps and long pretraining. At the target clone, the population instability is carried entirely by inputs on which the two ReLU gates disagree. For units at angle $\delta$ these inputs form a wedge of probability $\delta/\pi$, and weight decay shrinks the angle exponentially during pretraining. On every other input both units receive the same random linear update, which contracts their separation in conditional expectation. Bounding the cumulative probability of sampling the wedge along the exact online recursion, without a diffusion approximation, shows that recovery with fixed probability from an identical source-gradient-flow checkpoint, within $e^{c/\eta}$ updates, requires $Nb\gtrsim e^{\lambda T}$ target samples and batch size $b\gtrsim\eta e^{\lambda T}$, where $N$ counts updates and $\lambda$ is the weight decay. In simulations, recovery is approximately a function of the disagreement budget $b\delta/\eta$ and saturates in the horizon.
\end{abstract}

\maketitle

\section{Introduction}\label{sec:intro}

A network trained on one task can learn a second task more slowly than the same network trained from its original initialization. This loss of plasticity is documented in continual and reinforcement learning~\citep{dohare2024,lyle2023,nikishin2022,sokar2023}, and related difficulties arise when warm-starting~\citep{ash2020}. Theoretical accounts often reason with the population gradient field, for example through manifolds of cloned (positively proportional) neurons and their stability~\citep{joudaki2026} or the loss of curvature directions~\citep{lewandowski2023}. Training, however, uses finite batches. We ask whether the population field predicts how long finite-batch stochastic gradient descent (SGD) needs to adapt after pretraining.

Over a fixed horizon, small-step SGD tracks gradient flow (GF) under standard stochastic-approximation assumptions~\citep{benaim1999,kushner2003}. Exactly cloned neurons remain cloned under SGD~\citep{joudaki2026}. Gradient noise can even attract SGD to GF-unstable invariant sets such as sets of identical neurons, and smooth local coefficients at the set decide the attraction~\citep{chen2023,typeii}. Finite pretraining instead leaves distinct neurons at a small angle, and the population field predicts that they separate. For ReLU units the destabilizing term lives on the inputs where the two gates disagree, and pretraining makes these inputs rare. We prove that once pretraining exceeds order $\log(b/\eta)$, SGD stays near positively proportional neurons with high probability for a time exponential in $1/\eta$: target SGD is unlikely to sample a disagreeing input before their probability shrinks further.

\begin{figure*}[t]
\centering\includegraphics[width=\textwidth]{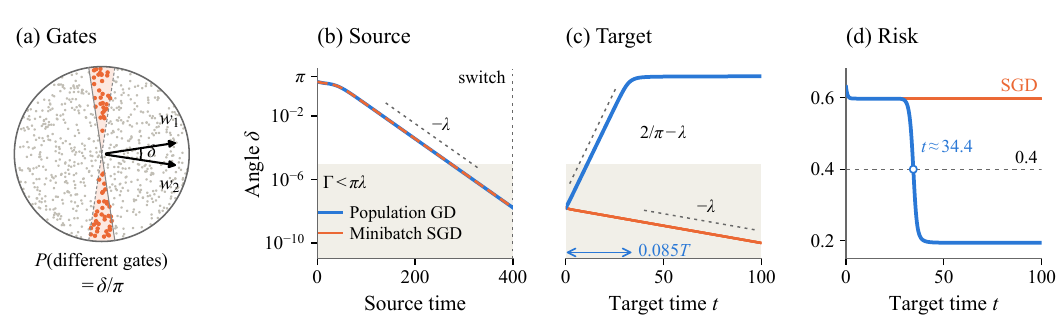}
\caption{\textbf{Mechanism on one trajectory} ($T=400$, $\eta=10^{-3}$, $b=16$). (a) Two gates at angle $\delta$ disagree on wedges of probability $\delta/\pi$. (b) Source training contracts the angle at rate $\lambda$. (c,d) From the source-SGD checkpoint, target GD (blue; Euler steps of GF) separates the units and crosses risk $0.4$ near the GF delay $\lambda T/(2/\pi-\lambda)$ (arrow), while four target SGD streams (orange) sample no disagreement and stay at risk ${\approx}0.6$. (c) shares the vertical axis of (b); dashed: rates of Section~\ref{sec:mechanism}; shading: $\budget<\pi\lambda$, under one expected disagreement (Poisson approximation).}
\label{fig:mechanism}
\end{figure*}

To make this precise, we work in a two-unit, bias-free ReLU network with inputs uniform on a disk. It is a minimal model in the tradition of small networks with analyzable training dynamics~\citep{li2026two,saad1995,tian2017,xu2023,yehudai2020}. The source label $\relu(x_1)$ drives the units toward positive proportionality, a controlled pretraining history. The radial target $\norm x$ rewards separated units: proportional units have target risk at least $0.59$, and target GF converges to risk below $0.27$, so reaching risk $0.4$ certifies separation. Without biases, the angle between the units determines gate disagreement. Rotational symmetry makes the population dynamics explicit, and bounded inputs bound sample gradients on bounded parameter sets. Both phases use ordinary online SGD, without tying, injected noise or balanced initialization. After source training for time $T$, GF in both phases recovers in time $\lambda T/(2/\pi-\lambda)+O(1)$ (about $0.085\,T$ at $\lambda=1/20$). SGD in both phases fails with high probability throughout a horizon of order $e^{c/\eta}$ once $T\gtrsim\log(b/\eta)$, which is exponential in $T$ when $T\lesssim1/\eta$ (Theorem~\ref{sgd:main}, Corollary~\ref{sgd:window}, Figure~\ref{fig:mechanism}).

\paragraph{Mechanism.}
Two ReLU units at angle $\delta$ disagree about which inputs are active only on wedges of probability $\delta/\pi$ (Figure~\ref{fig:mechanism}a). At a balanced \emph{target clone}, a stationary point with two positively proportional units, the destabilizing part of the population field comes entirely from this region (Proposition~\ref{prop:boundary}). On every other input, the same random linear map acts on both units, and an auxiliary process driven by the original independent batches contracts the separation of the two units in conditional expectation. Summing the disagreement probability along this contracting process bounds the probability of sampling a disagreeing input before the network leaves the proportional neighborhood. The bound is a constant multiple of $b\,d/\eta$ uniformly in the horizon, where the \emph{projective separation} $d=\norm{w_1/a_1-w_2/a_2}$ measures the distance from positive proportionality ($d=2\sin(\delta/2)$ for balanced units). A wedge of probability $\delta/\pi$ needs about $\pi/\delta$ samples to be hit. Before the first hit the separation contracts at a rate of order $\lambda$, so all future batches together contain about $b\delta/(\pi\lambda\eta)$ disagreeing samples in expectation. We call
\[
 \budget=\frac{b\,\delta}{\eta}
\]
the \emph{disagreement budget}. Source GF contracts the angle at rate $\lambda$, and source SGD contracts $d$ in expectation at a rate $\gamma>0$ of the same order. Warm SGD therefore fails with high probability once $Cb e^{-\gamma T}/\eta$ is small, i.e., after pretraining of order $\log(b/\eta)$. Both contractions come from weight decay: without it the balances $a_i^2-\norm{w_i}^2$ are conserved under GF~\citep{du2018} and the source angle equation loses its linear term, so the angle decays only like $1/T$.

\paragraph{Contributions.}
\begin{itemize}\itemsep0pt
\item \textbf{A two-phase separation theorem for online SGD} (Theorem~\ref{sgd:main}, Corollaries~\ref{sgd:window} and~\ref{sgd:fixedT}). On one event of high probability, fresh SGD succeeds in bounded time and warm SGD fails throughout a horizon of order $e^{c/\eta}$, uniformly on an explicit set of initializations with Gaussian probability above $1.06\%$. For each fixed pretraining duration, small-step SGD recovers as GF predicts, so the failure occurs only in the joint limit of small $\eta$ and growing $T$. The central estimate, a first-disagreement coupling with a summable gate hazard, is proved in full in Section~\ref{app:targetSGD}.
\item \textbf{Identical-checkpoint separation and sampling requirements} (Corollary~\ref{sgd:samecheckpoint}). From a source-GF checkpoint, target GF recovers in $O(T)$ time, whereas recovery by online SGD with fixed probability over horizons of at most order $e^{c/\eta}$ updates requires $Nb\gtrsim e^{\lambda T}$ target samples and $b\gtrsim\eta e^{\lambda T}$.
\item \textbf{Simulations of the two-unit model} in which recovery tracks $\budget$ across batch sizes, step sizes and checkpoint angles $\delta_0$ from $10^{-5}$ to $10^{-2}$, and saturates in the horizon when weight decay contracts the separation.
\end{itemize}

\section{Related Work}\label{sec:related}

Our result connects plasticity loss, stochastic attraction to invariant sets, neuron alignment and stochastic approximation.

\paragraph{Plasticity and redundant representations.}
Plasticity loss has been linked to dormant units, rank collapse, curvature and weight growth~\citep{dohare2024,kumar2021,lewandowski2023,lyle2023,lyle2024,sokar2023}, and representation and kernel ranks can mispredict trainability~\citep{wang2026}. \citet{joudaki2026} establish clone invariance across broad architectures and optimizers, examine stability, and study perturbations that help restore learning. Invariance explains why an exact clone persists; we bound how long SGD needs to separate the nearby, distinct neurons that finite pretraining produces, through the cumulative probability of sampling inputs that distinguish them.

\paragraph{Stochastic attraction to invariant sets.}
\citet{chen2023} show how gradient noise can attract SGD to invariant subnetworks that are unstable under GF, including permutation-invariant sets of identical neurons. Type-II saddle analysis studies random matrix products and probabilistic stability when gradient noise vanishes at the saddle~\citep{typeii}; constructed examples also show convergence of constant-step SGD to local maxima~\citep{ziyin2022}. In these analyses stability is decided by local coefficients at the invariant set, such as curvatures of the loss and of the noise covariance, and the mean-curvature prediction is recovered as the step size decreases at fixed coefficients~\citep{typeii}. Here the destabilizing term lives on a gate-boundary region that shrinks with the separation (Section~\ref{sec:mechanism}). At the clone the averaged fixed-input linearization contracts at rate $\lambda$, independently of $\eta$, so such criteria predict capture. Recovery instead depends on whether the checkpoint lies within angle of order $\eta/b$ of the clone; pretraining with weight decay gets there after time of order $\lambda^{-1}\log(b/\eta)$. We bound the cumulative probability of sampling the gate-disagreement region along the exact online recursion, without a diffusion approximation. It yields a hitting-time bound that separates SGD from GF in the joint limit of vanishing step size and growing pretraining duration.

\paragraph{Alignment and population escape.}
Weight decay can favor homogeneous neurons~\citep{ziyin2024}; small initialization can lead to early ReLU alignment~\citep{maennel2018}; and noise can drive motion along symmetry directions~\citep{ziyin2024noise}. Neuron duplication also creates critical subspaces~\citep{simsek2021}. Plateaus near such singularities and subsequent specialization are classical~\citep{fukumizu2000,saad1995,wei2008}, and online SGD from a random start in high dimension spends most of its samples escaping the region of small correlation with the signal~\citep{benarous2021}. \citet{li2026two} quantify approach to a bisector saddle and escape under population GF in a two-neuron teacher--student model. Their logarithmic population escape law is our comparator; finite-batch SGD from the same checkpoint needs exponentially many samples in the source-training time.

\paragraph{Stochastic approximation and saddle avoidance.}
Effective SGD dynamics can differ from GF in high-dimensional critical scalings~\citep{benarous2022}, and stochastic modified equations refine fixed-horizon tracking for smooth losses~\citep{li2017sme}. Avoidance of unstable points requires suitable excitation and regularity~\citep{brandiere1996,mertikopoulos2020,pemantle1990}; perturbation-based escape results cover smooth saddles~\citep{ge2015,jin2017}, with related results for classes of nonsmooth objectives~\citep{bianchi2021,davis2021}. For Gaussian inputs, ReLU population risks have closed-form gradients~\citep{tian2017}; the gate boundary enters population Hessians through Dirac terms that almost every sample Hessian of a piecewise-linear unit misses~\citep{zhong2017}, and empirical and population landscapes can differ~\citep{jin2018}. We quantify the consequence of such a boundary term when its supporting region shrinks during training: fixed-horizon tracking remains valid, but its recovery prediction fails once pretraining exceeds order $\log(b/\eta)$.

\section{Setting and Main Results}\label{sec:setup}

\paragraph{Model, tasks, and training.}
Let $X$ be uniform on the disk $\{x\in\R^2:\norm{x}\le2\}$ and
\begin{equation*}
 f_\theta(x)=a_1\relu(w_1^\top x)+a_2\relu(w_2^\top x),
\end{equation*}
with $\relu(z)=\max\{z,0\}$ and all six parameters $\theta=(a_1,w_1,a_2,w_2)$ trainable. The source and target labels are
\[
 y_s(x)=\relu(x_1),\qquad y_t(x)=\norm{x}.
\]
The source depends on one direction; the target is radial and rewards opposed units. For a label $y$, the population risk and training objective are
\[
 R_y(\theta)=\tfrac12\E(f_\theta(X)-y(X))^2,\quad
 F_y=R_y+\tfrac\lambda2\norm{\theta}^2,
\]
with a fixed weight decay $\lambda\in(0,1/5]$ (Appendix~\ref{app:lambda}). Numerical values are given at $\lambda=1/20$, the value used in the simulations. Training minimizes $F_y$; adaptation is measured by $R_y$, with subscripts $s,t$ for $y_s,y_t$. SGD with batch size $b\ge1$ and step size $\eta>0$ uses independent fresh inputs at every update:
\begin{equation}\label{sgd:update}
 \theta_{n+1}=(1-\eta\lambda)\theta_n-
 \frac\eta b\sum_{j=1}^b\nabla\ell_y(\theta_n;X_{n,j}),
\end{equation}
where $\ell_y(\theta;x)=\frac12(f_\theta(x)-y(x))^2$ and $\relu'(0)=0$. Physical time is $n\eta$. The population comparator is gradient flow (GF) $\dot\theta=-\nabla F_y(\theta)$.

The disk keeps the population geometry of Gaussian inputs while bounding sample gradients. Writing $X=R\hat X$ with uniform direction $\hat X$, the model and both labels are positively homogeneous of degree one, so the squared loss depends on the radial law only through $\E R^2$. Both this disk and $N(0,I_2)$ have $\E R^2=2$ and therefore the same population risk and GF equations; bounded inputs are used in the stochastic confinement argument.

\paragraph{Initialization.}
Write $r_i=\norm{w_i}$, $\omega_i=a_ir_i$, $B=\omega_1+\omega_2$, and let $\vartheta_i\in(-\pi,\pi)$ be the angle of $w_i$ from $e_1$. Let $\calD$ be the set of parameters in which both neurons have positive heads and lie on opposite sides of the source direction,
\[
 a_i>0,\ \ 0<\omega_i<1,\ \ \vartheta_1>0>\vartheta_2,\ \ \vartheta_1-\vartheta_2<\pi,
\]
together with its neuron exchange, and $\calD_0=\calD\cap\{B<3/4\}$. Appendix~\ref{app:explicitK} constructs a compact $K\subset\calD_0$ with explicit positive parameter and angular margins and proves $\Pp_{N(0,I_6)}(K)>0.01064$, close to $\Pp(\calD_0)\approx0.01083$.

\paragraph{A risk gap.}
Put $q=2/\pi$. If $(a_1,w_1)=c(a_2,w_2)$ with $c>0$, the predictor is a single ReLU and has target risk at least $R_{\rm col}=1-q^2\approx0.5947$. Target GF from $K$ converges to opposed units with risk $R_{\rm good}=1-2q^2+2\lambda^2<0.27$ (${\approx}0.1944$ at $\lambda=1/20$). Thus
\begin{equation}\label{sgd:riskgap}
 R_{\rm good}<0.4<R_{\rm col},
\end{equation}
and reaching risk $0.4$ requires leaving a neighborhood of the positively proportional family. Every $\theta_0\in K$ starts with target risk above $0.5928$; Appendices~\ref{app:targetpopulation} and~\ref{app:explicitK} establish these risk bounds.

\paragraph{Fresh and warm runs.}
Fix $\theta_0\in K$. The \emph{fresh} run applies target SGD from $\theta_0$. The \emph{warm} run trains on the source for physical time $T$ ($\lfloor T/\eta\rfloor$ updates) and then on the target with fresh batches, using the same $b$ and $\eta$. Let $\tau_f$ and $\tau_w(T)$ be the first target times at which population risk is at most $0.4$ (multiples of $\eta$ for SGD), and let $\tau_{\rm GF}(T)$ be the same time when both phases use gradient flow. Probabilities are over sample streams, conditional on $\theta_0$; the bounds below hold for any coupling of the fresh and warm streams.

\begin{theorem}[Linear GF delay versus metastable SGD failure]\label{sgd:main}
\textbf{(a) Gradient flow.} Uniformly over $\theta_0\in K$, fresh target GF reaches risk $0.4$ within a fixed time, and the warm GF hitting time satisfies
\begin{equation}\label{sgd:gfcontrast}
 \tau_{\rm GF}(T)=\frac{\lambda}{2/\pi-\lambda}\,T+O(1).
\end{equation}
\textbf{(b) Finite-batch SGD.} Fix $\varepsilon\in(0,1)$. There are constants $L,t_0,C,c_*,\gamma,\eta_0>0$, depending on the model, $\lambda$, $K$ and $\varepsilon$ but not on the batch size, such that for every $b\ge1$, $\theta_0\in K$, $0<\eta\le\eta_0$, $T\ge L+1$ and $H\ge0$, with conditional probability at least
\begin{equation}\label{sgd:combined}
 1-\varepsilon-\frac{Cb}{\eta}e^{-\gamma T}
 -C\Bigl(1+\frac{T+H}{\eta}\Bigr)e^{-c_*/\eta},
\end{equation}
the fresh run succeeds, $\tau_f\le t_0$, and the warm run fails throughout the horizon, $\tau_w(T)>H$.
\end{theorem}

Here $L$ is a source burn-in time, $\gamma$ the contraction rate of the source projective separation and $c_*$ a confinement rate. Any $\gamma<\lambda$ is admissible, at the cost of larger $L,C$ and smaller $c_*,\eta_0$, and we fix some $\gamma\in[\lambda/2,\lambda)$. At the source minima the averaged first-order linearization of the projective update, with both gates evaluated at one unit, is $I-\eta\lambda(I+e_1e_1^\top)\preceq(1-\eta\lambda)I$, and the disagreement correction is absorbed by shrinking the neighborhood and the step size (Appendix~\ref{app:sourceonestep}). Through a Markov estimate after the burn-in, $\varepsilon/2$ bounds the probability that the source phase drifts to an unfavorable allocation $\log(a_1/a_2)$, the relative scale of the two units along the clone family. A longer burn-in makes it smaller. The second term bounds, up to constants, $b\,\E[d_T]/\eta$ on the good source event, where $d_T$ is the projective separation at the checkpoint (Section~\ref{sec:proofmap}). The last term bounds excursions of the shared parameters (amplitude, balance) from the neighborhoods used in the proof.

\begin{corollary}[Metastable window]\label{sgd:window}
Take $\varepsilon=0.005$, fix $b$, and fix $C_0$ with $\gamma C_0>1$ (for instance $C_0>2/\lambda$). For all sufficiently small $\eta$, uniformly over $\theta_0\in K$, each prescribed source duration and horizon with
\[
 C_0\log(1/\eta)\le T\le e^{c_*/(2\eta)},\qquad H\le e^{c_*/(2\eta)}
\]
gives $\tau_f\le t_0$ and $\tau_w(T)>H$ with conditional probability at least $0.99$. In particular, for constants $c_0,c_1>0$ with $c_0c_1<c_*/2$ and $C_0\log(1/\eta)\le T\le c_0/\eta$ the warm run fails up to the exponential horizon $H=e^{c_1T}$.
\end{corollary}
\begin{proof}
For $T\ge C_0\log(1/\eta)$, the second term in \eqref{sgd:combined} is at most $Cb\eta^{\gamma C_0-1}\to0$. The last term is at most $C(1+2\eta^{-1}e^{c_*/(2\eta)})e^{-c_*/\eta}\to0$. Eventually $C_0\log(1/\eta)\ge L+1$. For the particular case, $T\le c_0/\eta\le e^{c_*/(2\eta)}$ and $H=e^{c_1T}\le e^{c_1c_0/\eta}\le e^{c_*/(2\eta)}$ for small $\eta$.
\end{proof}

Over Gaussian initialization the guarantee holds with probability at least $0.99\,\Pp(K)>1\%$, and smaller $\varepsilon$ raises the conditional confidence to any level below one.

\paragraph{Two regimes.}
The guaranteed horizon $e^{c_*/(2\eta)}$ does not depend on $T$; pretraining enters through the term $Cbe^{-\gamma T}/\eta$. The horizon comes from exponential concentration of the shared parameters iterated over blocks of fixed physical time, as in small-noise metastability~\citep{freidlin2012}. For fixed pretraining, warm SGD instead recovers in bounded time, as gradient flow predicts:

\begin{corollary}[Fixed pretraining duration]\label{sgd:fixedT}
For each $\theta_0\in K$ and $T\ge0$ there are $t_1,C,c,\eta_2>0$, independent of $b$, such that $\Pp(\tau_w(T)\le t_1\mid\theta_0)\ge1-Ce^{-c/\eta}$ for every $b\ge1$ and $0<\eta\le\eta_2$.
\end{corollary}

Corollary~\ref{sgd:fixedT} cannot hold with a step-size threshold and constants uniform in $T$, since by Corollary~\ref{sgd:window} warm SGD fails once $T\ge C_0\log(1/\eta)$. In the joint sweep (Appendix~\ref{app:joint}, Figure~\ref{fig:joint}), the source time at which warm recovery is lost grows as $\eta$ decreases.

\section{The Instability Lives on the Gate Boundary}\label{sec:mechanism}

This section locates the population instability behind Theorem~\ref{sgd:main} on the gate boundary. Source gradient flow from $K$ converges (Proposition~\ref{prop:source}, Appendix~\ref{app:sourcepopulation}) to minima in which both units stay active, with both neurons along $e_1$, balanced parameters $a_i=\norm{w_i}$, and total amplitude $B_s=1-2\lambda$. The target has a proportional stationary family with amplitude $B_t=2(q-\lambda)$. Put $A=q-\lambda$ and $s_*=\sqrt A$, start from $a_i=s_*$, $w_i=s_*e_1$, and split the hidden vectors transversally:
\[
 w_1=(s_*,\varphi),\qquad w_2=(s_*,-\varphi).
\]
The \emph{common-gate} field applies the shared gate $\ind_{x_1>0}$ to both units: its predictor is $\ind_{x_1>0}\sum_ia_iw_i^\top x$, and its sample update applies the corresponding residual and gate to each full neuron. Samplewise, it coincides with SGD whenever the two gates agree.

\begin{proposition}[Instability carried by gate disagreements]\label{prop:boundary}
Let $V=-\nabla F_t$, let $V_{\rm com}$ be the common-gate mean field, and let $v=\frac12(V_{w_{12}}-V_{w_{22}})$, with $V_{w_{i2}}$ the second coordinate of the $w_i$-component of $V$, and define $v_{\rm com}$ likewise from $V_{\rm com}$. Then
\[
 v(\varphi)=A\varphi+O(\varphi|\varphi|),\qquad
 v_{\rm com}(\varphi)=-\lambda\varphi.
\]
Their per-sample integrands agree outside a disagreement region of probability $2\arctan(|\varphi|/s_*)/\pi$, so their mean difference $q\varphi+O(\varphi|\varphi|)$ is carried entirely by that region.
\end{proposition}
\begin{proof}
The separation is $\delta=2\arctan(|\varphi|/s_*)$, and the exact population angular equation (Appendix~\ref{app:populationequations}) gives $\dot\delta=2s_*^2g(\delta)(1+O(\varphi^2))$ with $g(\delta)=\delta/2+O(\delta^2)$, while radial drift is $O(\varphi^2)$. For the common gate, $U=\sum_ia_iw_i=2Ae_1$, and rotational symmetry gives $\E[(\ind_{X_1>0}U^\top X-\norm X)\ind_{X_1>0}X]=U/2-qe_1=-\lambda e_1$, with zero transverse component, so transverse coordinates decay at rate $\lambda$.
\end{proof}

At $\lambda=1/20$, $A\approx0.587$, so the boundary contribution $q\approx0.637$ reverses the sign of the common-gate rate $-0.05$; Figure~\ref{fig:mechanism}(c) shows both rates on actual trajectories.

\paragraph{Sample and population linearizations.}
For almost every fixed input, a small enough splitting leaves both gates unchanged, so the fixed-input linearization of the sample gradient averages to the common-gate contraction. The sign reversal comes from inputs in a region that moves and shrinks with the splitting. For smooth sample losses with integrable derivative bounds, differentiation and expectation commute, so averaging the sample Jacobian recovers the population Jacobian; the moving gate boundary breaks this interchange. The population rate $A$ is recovered only through samples of probability $\delta/\pi$, and one such sample changes the splitting by order $\eta/b$, a relative change of order $\eta/(b\delta)$. Applied at the clone, a criterion built from the averaged fixed-input linearization predicts contraction and one built from the population linearization predicts escape; both are fixed by the clone and cannot see how long pretraining lasted.

\paragraph{Collapse without inactive neurons.}
Each neuron activates on half the distribution for every $\delta$, while the probability that the two gates differ vanishes as $\delta/\pi$. Activation frequency therefore misses the loss of distinct responses: sampled gradients act on both neurons through the same map, rescaling the shared output but rarely separating the units.

\paragraph{Amplitude adjustment at the task switch.}
The source and target clone amplitudes differ: $B_s=1-2\lambda<B_t=2(q-\lambda)$ ($0.9$ and ${\approx}1.173$ at $\lambda=1/20$). On the balanced proportional family, the target equation is $\dot B=B(B_t-B)$, so the task switch adjusts the shared amplitude while preserving proportionality. Throughout this transient, $B<2q$, the range in which the common-gate estimate \eqref{sgd:transversenegative} contracts transverse separation. Adapting the output scale thus lowers risk toward the target-clone value without creating the angular separation needed to cross $0.4$.

\paragraph{The linear population delay.}
Population source training gives $\delta_s(T)=C_se^{-\lambda T}(1+o(1))$ with $C_s$ bounded above and away from zero on $K$ (Appendix~\ref{app:sourceescape}). Near the clone family the target angle grows at rate $A$ after a bounded amplitude transient, so
\[
 \tau_{\rm GF}(T)=\frac1A\log\frac1{\delta_s(T)}+O(1)=\frac\lambda AT+O(1).
\]
This is the expected logarithmic dependence on initial asymmetry for escape from a symmetric plateau~\citep{saad1995,wei2008}.

\section{Identical Checkpoints and Sample Sizes}\label{sec:checkpoint}

Theorem~\ref{sgd:main} uses SGD in both phases. The next result starts both target optimizers from one source-GF checkpoint, so only target sampling differs.

\begin{corollary}[Identical checkpoint]\label{sgd:samecheckpoint}
Let $\theta_s^{\rm GF}(T)$ be the source GF checkpoint from $\theta_0\in K$, and let $\theta_n^{\rm SGD}$ be target SGD started at $\theta_s^{\rm GF}(T)$. There are $C,c,T_1,\eta_1>0$ such that for all $T\ge T_1$, $0<\eta\le\eta_1$, $b\ge1$, integers $N\ge1$ and $\theta_0\in K$,
\begin{equation}\label{sgd:samecheckpointbound}
\begin{aligned}
 &\Pp\Bigl(\min_{n\le N}R_t(\theta_n^{\rm SGD})\le0.4\Bigr)\\
 &\quad\le Cb e^{-\lambda T}\min\{N,\eta^{-1}\}+C(N+1)e^{-c/\eta},
\end{aligned}
\end{equation}
while target GF from the same checkpoint reaches $0.4$ at the time given by \eqref{sgd:gfcontrast}.
\end{corollary}
\begin{proof}
Source GF gives positive limiting amplitudes, angle $\le Ce^{-\lambda T}$ (Lemma~\ref{lem:slow}) and balance error $O(e^{-2\lambda T})$, so $\norm{w_1/a_1-w_2/a_2}\le Ce^{-\lambda T}$. By \eqref{sgd:coordinatecompare} in Appendix~\ref{app:composition}, the target projective coordinate obeys the same bound, and for large $T$ the checkpoint lies in the target neighborhood of Proposition~\ref{sgd:targettheorem}. Apply Proposition~\ref{sgd:targettheorem}.
\end{proof}

\paragraph{Necessary sample and batch sizes.}
If target SGD from $\theta_s^{\rm GF}(T)$ reaches $0.4$ within $N$ steps with probability at least $p$, and $C(N+1)e^{-c/\eta}\le p/2$, then \eqref{sgd:samecheckpointbound} forces $b\min\{N,\eta^{-1}\}\ge c_pe^{\lambda T}$ with $c_p=p/(2C)$. That is,
\[
 Nb\ge c_pe^{\lambda T},\qquad b\ge c_p\,\eta\, e^{\lambda T}.
\]

The first condition bounds the total number of target samples. The second persists as $N$ grows while the confinement error is small: later batches meet a smaller wedge, so at fixed step size a larger batch can satisfy this condition and a longer run cannot. Neither condition is sufficient, since a sampled disagreement need not trigger escape.

\paragraph{An explicit joint scaling.}
Fix $b$ and take a source-GF checkpoint at $T=c_T\log(1/\eta)$ with $c_T>1/\lambda$. For every $N$ the disagreement term of \eqref{sgd:samecheckpointbound} is at most
\[
 C\frac{b}{\eta}e^{-\lambda T}=Cb\,\eta^{\lambda c_T-1}\longrightarrow0,
\]
and on a horizon $H=\eta^{-c_H}$ with fixed $c_H>0$ the confinement term also vanishes. The SGD recovery probability on each such horizon thus tends to zero uniformly on $K$, while GF recovers in $O(\log(1/\eta))$ time.

\section{Proof Overview}\label{sec:proofmap}

The proof treats the phases in turn: source SGD contracts the projective separation, and the target analysis bounds the probability of ever sampling a disagreement. Appendix~\ref{app:notation} collects the notation.

\paragraph{Why the proof uses projective separation.}
Angle alone does not measure distance from positively proportional full neurons: with $m_i=r_i/a_i$, $d^2=(m_1-m_2)^2+2m_1m_2(1-\cos\delta)$, so $d=2\sin(\delta/2)$ under exact balance, while small balance errors can dominate a much smaller angle. After source GF, balance errors are $O(e^{-2\lambda T})$ and the angle is of order $e^{-\lambda T}$, so $d\asymp\delta$. For source SGD we work with $d$ directly, which allows arbitrary initialization in $K$ without imposing balance.

\paragraph{Source contraction proportional to the separation.}
Appendix~\ref{app:sourceSGD} treats the source phase. With positive heads, let $u_i=w_i/a_i$ and $d=\norm{u_1-u_2}$. Positive proportionality, $u_1=u_2$, is preserved by every sample update. Near the source minima, with $\mathcal F_n$ the $\sigma$-field of the first $n$ batches, the exact projective update gives (Appendix~\ref{app:sourceonestep})
\[
 \E[d_{n+1}\mid\mathcal F_n]\le(1-\gamma\eta)d_n,
\]
with no additive term of order $\eta$.
Near the source direction, both the source label and network output are small on disagreeing inputs. An $L^1$ estimate keeps a factor $d$ in the expected wedge correction, so it can be absorbed into the contraction. The radial target label $\norm x$ has no small factor on the same wedge, so its contribution can reverse the population transverse stability. A fixed population burn-in and Lemma~\ref{sgd:tracking} carry SGD from $K$ into this neighborhood (Appendix~\ref{app:entry}). The allocation, measured here by the log head ratio, then moves by at most $C\eta d_n$ per step, so its total motion is summable; a long enough burn-in makes losing the allocation margin cost at most $\varepsilon/2$, the origin of the fixed allowance $\varepsilon$ (Appendices~\ref{app:sourceconfine} and~\ref{app:entry}). This gives a good event $E_T$ with $\Pp(E_T^c)\le\varepsilon/2+C(1+T/\eta)e^{-c/\eta}$ and $\E[d_T\ind_{E_T}\mid\theta_0]\le Ce^{-\gamma T}$.

\paragraph{Target phase.}
Section~\ref{app:targetSGD} proves the target obstruction, Proposition~\ref{sgd:targettheorem}. An auxiliary process applies the gate of the shared hidden vector to both units, so both are multiplied by the same random matrix, and it coincides with SGD up to the first batch with a disagreement. Its target projective separation $\Delta$ (defined in Section~\ref{app:targetSGD}) satisfies
\[
 \E[\Delta_{n+1}\mid\mathcal F_n]\le(1-\eta\lambda/4)\Delta_n
\]
in a tube around the clone family, for every batch size, because on the proportional family the mean common-gate matrix satisfies $e^\top(\E G)e\le-\lambda\norm e^2$ for $e\perp h$ whenever $B\le2q$, a range that covers the amplitude transient from $B_s$ to $B_t$. A batch contains a disagreement with probability at most $Cb\Delta_n$, so the summed hazard is at most $Cb\Delta_0\min\{N,4/(\lambda\eta)\}$. Confinement of amplitude and balance over blocks of fixed physical time adds $C(N+1)e^{-c/\eta}$ and sets the exponential horizon.

\paragraph{Composition.}
At the checkpoint $\Delta_0\le Cd_T$ (Appendix~\ref{app:composition}); conditioning on it and using the source bound gives Theorem~\ref{sgd:main}(b). Fresh success follows from finite-time tracking of target GF with a strict risk margin.

\section{Proof of the Target Obstruction}\label{app:targetSGD}

The target half of Theorem~\ref{sgd:main} rests on one estimate: from a checkpoint near the proportional family, target SGD reaches risk $0.4$ only by sampling a gate disagreement, by a projective or allocation exit of probability $O(\Delta_0)$, or by an exponentially rare excursion of amplitude and balance. This section proves it, except for the confinement of amplitude and balance, which is a standard block argument (Appendix~\ref{app:bridge}).

\paragraph{Coordinates.}
Recall $q=2/\pi$, $B_t=2(q-\lambda)$ and $R_{\rm col}=1-q^2=0.5947\ldots$. For the full neuron vectors $z_i=(a_i,w_i)\in\R^3$, define
\begin{gather*}
 h=\frac{z_1+z_2}{\sqrt2}=(h_a,w),\qquad
 d_z=\frac{z_1-z_2}{\sqrt2},\\
 \rho=\frac{h^\top d_z}{\norm{h}^2},\qquad e=d_z-\rho h,
\end{gather*}
and $r=\norm{w}$, $S=1+\rho^2$, $k=h_a^2-r^2$, $B=Sh_a r$ (equal to $\omega_1+\omega_2$ on the proportional family). Here $e\perp h$. The proportional family is $e=0$, $\abs\rho<1$, $h_a,r>0$, and its target stationary manifold is $k=0$, $B=B_t$. The allocation $\rho$ records the relative scale of the two units, $\Delta=\norm e$ is the target projective separation, and $d_z$ is distinct from the source projective distance $d$.

\begin{proposition}[Finite-horizon target obstruction]\label{sgd:targettheorem}
Fix $0<\rho_0<\rho_1<\rho_2<1$, and recall $B_s=1-2\lambda<B_t$. There exist positive $\nu,\epsilon_0,\eta_{\rm t},C,c$ such that, whenever
\begin{gather*}
 h_a,r>0,\qquad\abs\rho\le\rho_0,\\
 \abs k+\abs{B-B_s}\le\nu,\qquad\norm e=\Delta_0\le\epsilon_0,
\end{gather*}
target SGD with any batch size $b\ge1$, $0<\eta\le\eta_{\rm t}$ and integer $N\ge1$ satisfies
\begin{align*}
 &\Pp\Bigl(\min_{0\le n\le N}R_t(\theta_n)\le0.4\Bigr)\\
 &\qquad\le Cb\Delta_0\min\{N,\eta^{-1}\}+C(N+1)e^{-c/\eta}.
\end{align*}
The constants are independent of $b,N,\Delta_0$ and the common direction.
\end{proposition}

The proof has three steps. An auxiliary process applies one gate to both units (\cref{app:commongate}); its projective separation contracts exactly in conditional expectation (\cref{app:projective}); and a stopping argument bounds the probability that actual SGD leaves the auxiliary process by the summed disagreement hazard (\cref{app:hazard}). \Cref{sec:targetconfine} states the confinement bound proved in Appendix~\ref{app:bridge}.

\subsection{An Exact Common-Gate Process}\label{app:commongate}

For $u=w/r$, let $g_\theta(x)=\ind_{w^\top x>0}$ and $U=a_1w_1+a_2w_2$. On the same sampled batch, define
\begin{gather*}
 f_{\rm aux}(x)=g_\theta(x)U^\top x,\\
 v=\frac1b\sum_j(f_{\rm aux}(X_j)-\norm{X_j})g_\theta(X_j)X_j,
\end{gather*}
\begin{gather*}
 G=-\lambda I-M(v),\qquad
 M(v)=\begin{pmatrix}0&v^\top\\v&0_{2\times2}\end{pmatrix},\\
 Q=I+\eta G,\qquad z_i^+=Qz_i.
\end{gather*}
Both units are multiplied by the same random matrix $Q$. If the two hidden gates agree on every batch example, their common gate also equals the gate of $w=(w_1+w_2)/\sqrt2$, and the auxiliary update is exactly the ordinary update \eqref{sgd:update}. The processes therefore couple until the first gate disagreement or the tube exit defined below.

Rotational symmetry gives the exact conditional mean
\begin{equation}\label{sgd:auxmean}
 \E[v\mid\theta]=U/2-qu,
\end{equation}
using $\E[\ind_{u^\top X>0}XX^\top]=I/2$ and $\E[\norm X\ind_{u^\top X>0}X]=(2/\pi)u$. Thus the auxiliary mean field is smooth on $r>0$, although the SGD update is not.

\subsection{Exact Finite-Step Projective Contraction}\label{app:projective}

Since $h^+=Qh$ and $d_z^+=Qd_z$, orthogonal projection gives exactly
\begin{equation}\label{sgd:projection}
\begin{gathered}
 e^+=P_{Qh}Qe,\qquad
 \rho^+-\rho=\frac{h^\top Q^2e}{\norm{Qh}^2},\\
 P_v=I-\frac{vv^\top}{\norm v^2},
\end{gathered}
\end{equation}
where $Q=Q^\top$ was used. On the proportional family, $U=Bu$ and $G_P:=\E G=-\lambda I+(q-B/2)M(u)$. For $e\perp h$, $h_a e_a+r u^\top e_w=0$, and therefore
\begin{equation}\label{sgd:transversenegative}
\begin{aligned}
 e^\top G_Pe&=-\lambda\norm e^2-2(q-B/2)\frac{h_a}{r}\, e_a^2\\
 &\le-\lambda\norm e^2\qquad(B\le2q).
\end{aligned}
\end{equation}
The radial target thus enters the common-gate dynamics only through the shared direction $u$, and the transverse part of the mean matrix is a pure contraction. Off the proportional family, the exact identity
\[
 U=Sh_a w+\rho(h_a e_w+e_aw)+e_ae_w
\]
implies $\norm{\E G-G_P}\le C_U\norm e$ on any fixed compact chart with positive lower bounds on $h_a,r$ and $\abs\rho\le\rho_2$. Choose a tube with $B\le B_t+\lambda/2<2q$ and projective radius $\epsilon$ so that $C_U\epsilon\le\lambda/2$.

Bounded samples give a bound $\norm G\le L_G$ for every possible batch in this tube; for example $\norm v\le4(\norm U+1)$. For $\eta\le\min\{(2L_G)^{-1},\lambda/(2L_G^2)\}$, \eqref{sgd:projection}--\eqref{sgd:transversenegative} imply
\begin{equation}\label{sgd:targetcontract}
\begin{aligned}
 \E[\norm{e^+}^2\mid\theta]&\le(1-\eta\lambda/2)\norm e^2,\\
 \E[\norm{e^+}\mid\theta]&\le(1-\eta\lambda/4)\norm e.
\end{aligned}
\end{equation}
Indeed, expand $\norm{Qe}^2$ and use $\norm{e^+}\le\norm{Qe}$, then Jensen. Also $h^\top e=0$ in \eqref{sgd:projection} gives
\begin{equation}\label{sgd:targetallocation}
 \abs{\rho^+-\rho}\le C\eta\norm e,
\end{equation}
because $Q^2=I+2\eta G+\eta^2G^2$ and $\norm{Qh}\ge\norm h/2$. Both estimates hold for every batch size, with constants independent of $b$.

\subsection{Confinement of the Shared Parameters}\label{sec:targetconfine}

The amplitude $B$ and the balance $k$ have stable mean dynamics on the proportional family, and the task switch moves $B$ from $B_s$ to $B_t$ along the logistic path $\dot B=B(B_t-B)$, which stays below $2q$. Appendix~\ref{app:bridge} defines a compact bridge tube for this transient, of fixed duration $\ell_0$, and a common tube $\{\sqrt{W_t}\le R_W/2\}$ with $W_t=k^2+(B-B_t)^2$ around $k=0$, $B=B_t$, both inside the chart of \cref{app:projective}, and shows with Lemmas~\ref{sgd:tracking} and~\ref{sgd:block} that
\begin{equation}\label{sgd:targetnormalexit}
 \Pp(\text{common-tube exit first by }N)\le C(N+1)e^{-c/\eta},
\end{equation}
where ``first'' means at or before the other stops defined below. This is the only source of the exponential horizon.

\subsection{Stopping, Gate Hazards, and Actual Population Risk}\label{app:hazard}

Let $n_0=\lceil\ell_0/\eta\rceil$ and let $\mathcal T_n$ be the compact bridge tube for $n<n_0$ and the common tube $\{\sqrt{W_t}\le R_W/2\}$ for $n\ge n_0$ (Appendix~\ref{app:bridge}); Lemma~\ref{sgd:block} is restarted at $n_0$ from the state reached at the end of the bridge. With $\tilde\theta_n$ the auxiliary process of \cref{app:commongate} (frozen after the stop, as defined below), let $\zeta_t$ be the first $n$ with $\tilde\theta_n\notin\mathcal T_n$, $\abs\rho\ge\rho_1$, or $\norm e\ge\epsilon$. All constants are chosen on a slightly larger compact chart, so the update which exits is still defined and obeys the one-step inequalities. Killing after the stop and using \eqref{sgd:targetcontract} give
\begin{equation}\label{sgd:targetkilled}
 \E[\norm{\tilde e_n}\ind_{n<\zeta_t}]
 \le\Delta_0(1-\eta\lambda/4)^n.
\end{equation}
The stopped squared norm $\norm{e_{n\wedge\zeta_t}}^2$ is a nonnegative supermartingale: the inequality holds also for the update which first exits. Ville's maximal inequality for this supermartingale, and Markov's inequality with \eqref{sgd:targetallocation}, yield
\begin{align}
 &\Pp(\text{projective exit first})\le\Delta_0^2/\epsilon^2,
 \label{sgd:projectiveexit}\\
 &\Pp(\text{allocation exit first})\le\frac{C\Delta_0}{\lambda(\rho_1-\rho_0)},
 \label{sgd:allocationexit}
\end{align}
with ``first'' as in \eqref{sgd:targetnormalexit}. For the latter, sum expected absolute increments using \eqref{sgd:targetkilled} and apply Markov's inequality.

In the compact tube the hidden vectors are perturbations of the nonzero, positively proportional vectors $(1\pm\rho)w/\sqrt2$. Their angle is at most $C_{\rm ang}\norm e$. A rotationally uniform sample disagrees with their two gates with exact probability equal to this angle divided by $\pi$.

\paragraph{Event inclusion.}
Retain the exit state at $\zeta_t$ and freeze the auxiliary process thereafter. This defines $\tilde\theta_n$ for all $n$ without extending the local chart; the stopped estimates are unchanged. Let $\mathcal G_n$ be the event that some sample of batch $n$ falls where the two hidden gates of $\tilde\theta_n$ disagree, and let $n_{\rm dis}$ be the first $n$ with $\mathcal G_n$. Actual SGD equals $\tilde\theta_n$ for every $n\le n_{\rm dis}\wedge\zeta_t$. If $n_{\rm dis}\ge N$ and $\zeta_t>N$, then $\theta_n=\tilde\theta_n$ lies in the tube for all $n\le N$, and its risk exceeds $0.4$ by the risk-gap argument below. Hence
\begin{align*}
 &\Bigl\{\min_{n\le N}R_t(\theta_n)\le0.4\Bigr\}\\
 &\qquad\subset\{n_{\rm dis}<N\wedge\zeta_t\}\cup\{\zeta_t\le N\}.
\end{align*}
Because $\{n<\zeta_t\}\in\mathcal F_n$ and the batch at step $n$ is independent of $\mathcal F_n$, $\Pp(\mathcal G_n\mid\mathcal F_n)\le bC_{\rm ang}\norm{\tilde e_n}/\pi$ on $\{n<\zeta_t\}$. With \eqref{sgd:targetkilled},
\begin{equation}\label{sgd:gatebudget}
\begin{aligned}
 \Pp(n_{\rm dis}<N\wedge\zeta_t)
 &\le\sum_{n<N}\E\bigl[\Pp(\mathcal G_n\mid\mathcal F_n)\ind_{n<\zeta_t}\bigr]\\
 &\le\frac{bC_{\rm ang}\Delta_0}{\pi}\sum_{n<N}(1-\eta\lambda/4)^n\\
 &\le\frac{bC_{\rm ang}\Delta_0}{\pi}\min\Bigl\{N,\frac{4}{\lambda\eta}\Bigr\}.
\end{aligned}
\end{equation}
The conditioning is on the past only: the hazard of step $n$ is computed for a fresh batch that is independent of $\mathcal F_n$, so no sample is ever conditioned on having avoided the wedge. The geometric sum is why the disagreement term saturates in $N$: contraction caps the total expected number of disagreeing samples at order $b\Delta_0/(\lambda\eta)$, however long the run.

\paragraph{Risk gap.}
Replace $d_z$ by $\rho h$ to obtain a proportional parameter $\theta_P$. The orthogonal change of coordinates gives $\norm{\theta-\theta_P}=\norm e$. Its predictor is a single ReLU, so its risk is at least $R_{\rm col}>0.4$ by \eqref{eq:clonegap}. Risk is Lipschitz on the compact tube. Reducing $\epsilon$ so that $L_R\epsilon<R_{\rm col}-0.4$ ensures that every state in the tube has risk above $0.4$. The event $\{\zeta_t\le N\}$ is bounded by \eqref{sgd:targetnormalexit}--\eqref{sgd:allocationexit}, and $\{n_{\rm dis}<N\wedge\zeta_t\}$ by \eqref{sgd:gatebudget}. The terms $C\Delta_0$ and $C\Delta_0^2$ are absorbed into $Cb\Delta_0\min\{N,\eta^{-1}\}$ since $N,b\ge1$, $\eta\le1$, $\Delta_0\le1$. This proves Proposition~\ref{sgd:targettheorem}.

\section{Simulations}\label{sec:experiments}

Simulations of the two-unit model check four consequences of the analysis: the gap appears from identical checkpoints; the recovery probability is approximately a function of the budget $\budget$ that controls the bound; the disagreement term saturates in the horizon, and so does observed recovery; and saturation depends on the contraction supplied by weight decay. All runs use fresh minibatches from the disk, double precision, the analytic population risk and, except for the weight-decay controls, $\lambda=0.05$. GD uses the same $\eta$ (Appendix~\ref{app:experiments} checks step halving). Here $\budget=b\delta_0/\eta$ with $\delta_0$ the angle at the task switch ($\budget_T$ after source time $T$). Appendix~\ref{app:experiments} gives details. The code is provided as ancillary files with this preprint.

\begin{figure}[t]
\centering\includegraphics[width=\columnwidth]{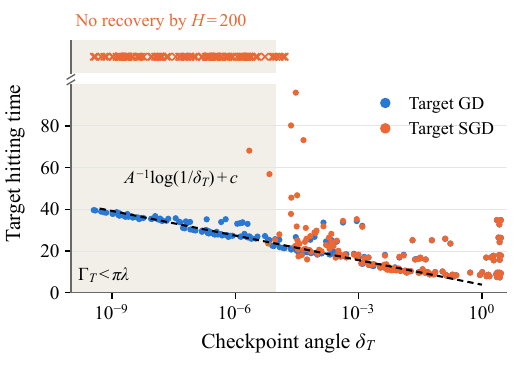}
\caption{\textbf{Identical checkpoints.} Target hitting times from shared source-SGD checkpoints ($T\in\{100,200,300,400\}$, $\eta=10^{-3}$, $b=16$). GD (blue) follows $A^{-1}\log(1/\delta_T)$ plus the median offset (dashed); SGD (orange) rarely recovers by $H=200$ (top strip) once $\budget_T<\pi\lambda$ (shaded).}
\label{fig:paired}
\end{figure}

\begin{figure*}[t]
\centering\includegraphics[width=0.9\textwidth]{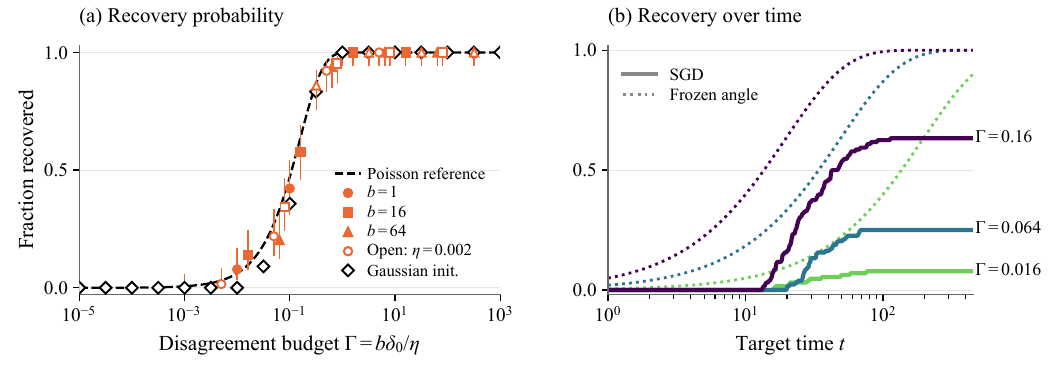}
\caption{\textbf{Recovery is organized by the disagreement budget.} (a) Recovery probability is approximately a function of $\budget=b\delta_0/\eta$ and lies at or below the Poisson reference $1-e^{-\budget/(\pi\lambda)}$ (dashed) up to sampling error. Orange: constructed checkpoints with 95\% Wilson intervals (filled: $\eta=0.01$); diamonds: checkpoints from Gaussian initializations. (b) Recovery saturates in the horizon (solid), far below the frozen-angle reference (dotted); three of the five long-horizon settings are shown.}
\label{fig:resolution}
\end{figure*}

\begin{table}[t]
\caption{\textbf{Identical checkpoints.} Number of the 64 saved initializations from which each target optimizer reaches risk $0.4$ by $H=200$ ($\eta=10^{-3}$, $b=16$). The source-GD columns are the protocol of Corollary~\ref{sgd:samecheckpoint}.}\label{tab:paired}
\centering\small\setlength{\tabcolsep}{4.5pt}
\begin{tabular}{@{}rcccc@{}}
\toprule
& \multicolumn{2}{c}{source SGD} & \multicolumn{2}{c}{source GD}\\
\cmidrule(lr){2-3}\cmidrule(lr){4-5}
$T$ & target SGD & target GD & target SGD & target GD\\
\midrule
0 & 64 & 64 & 64 & 64\\
100 & 64 & 64 & 64 & 64\\
200 & 59 & 64 & 61 & 64\\
300 & 20 & 64 & 19 & 64\\
400 & 7 & 64 & 7 & 64\\
\bottomrule
\end{tabular}
\end{table}

\paragraph{Identical checkpoints.}
We draw 64 Gaussian initializations conditioned on $\calD_0\supset K$, train each on the source with both SGD and GD, and pass each checkpoint to both target optimizers. Fresh target SGD ($T=0$) recovers from all 64 initializations. Target GD recovers from every checkpoint, with hitting time growing as $\log(1/\delta_T)/A$ (Figure~\ref{fig:paired}). Target SGD recovers from every checkpoint with $\budget_T\gtrsim1$, recovers late or only sometimes for $10^{-2}\lesssim\budget_T\lesssim1$, and has no observed recovery below $10^{-2}$. By $T=400$, 57 of the 64 source-SGD checkpoints have $\budget_T<0.04$; none of them samples a disagreement, and no run recovers after $t=100$ although GD recovers by $t\approx40$. All seven recoveries at $T=400$ have $\budget_T>1$ and start from near-opposite initializations (initial angle above $2.5$); at each, one unit has amplitude below $10^{-8}$, so its direction barely moves. Source-GD checkpoints, the protocol of Corollary~\ref{sgd:samecheckpoint}, give the same counts up to two runs at every $T$ (Table~\ref{tab:paired}).

\paragraph{Recovery across disagreement budgets.}
Beyond $1/\eta$ updates, the disagreement term of Proposition~\ref{sgd:targettheorem} is a multiple of $\budget$ at balanced checkpoints. Figure~\ref{fig:resolution}(a) shows that the recovery probability is approximately a function of $\budget$. It pools the 512 checkpoints with $T>0$ of the identical-checkpoint experiment (both source optimizers) with constructed balanced checkpoints at amplitude $0.9$, in which $\delta_0\in\{10^{-5},10^{-4},10^{-3},10^{-2}\}$, $b\in\{1,16,64\}$ and $\eta\in\{0.01,0.002\}$ vary independently (24 cells, 64 streams each, $H=200$). As a reference we use a Poisson count: if the angle decays as $\delta_0e^{-\lambda t}$ before the first disagreement, the expected number of disagreeing samples is $\budget/(\pi\lambda)$, and if each triggered recovery, recovery would have probability $1-e^{-\budget/(\pi\lambda)}$. The reference captures the location and overall shape of the observed transition, and it lies inside or above the 95\% interval of every cell. Every observed recovery follows a sampled disagreement, and most runs that sample one recover (Appendix~\ref{app:experiments}). At $\delta_0=10^{-4}$ and $\eta=0.01$, raising $b$ from $1$ to $64$ raises recovery from $8\%$ to $94\%$.

\paragraph{Saturation over long horizons.}
In five constructed settings ($\budget$ from $0.016$ to $0.64$), no recovery occurs after $t\approx112$ through $t=450$, about 20 to 32 times their GD hitting times. The frozen-angle reference, which keeps the initial angle and counts recovery at the first disagreement, exceeds $0.63$ by $t=200$ in every setting (Figure~\ref{fig:resolution}b). Recovery thus saturates in the horizon, mirroring the factor $\min\{N,\eta^{-1}\}$ in the bound.

\paragraph{Weight decay.}
Controls from constructed checkpoints change only the target phase (Appendix~\ref{app:controls}, Figure~\ref{fig:controls}). With $\lambda\in\{0.01,0.05\}$, no first disagreement is observed after $t\approx240$, and recovery plateaus through $H=1000$ at or below the Poisson reference. Without target weight decay the separation contracts only weakly and recovery keeps accumulating, so saturation comes from the contraction supplied by weight decay.

\section{Conclusion and Future Work}\label{sec:discussion}

In this model, population instability mispredicts finite-batch adaptation after pretraining because the inputs that separate near-clones become rarer as the common-gate dynamics contract their separation. Bounding their cumulative sampling probability gives hitting-time lower bounds and necessary sample and batch sizes. Warm SGD recovers as $\eta\to0$ at each fixed $T$ (Corollary~\ref{sgd:fixedT}) and fails once $T\gtrsim\log(b/\eta)$ (Theorem~\ref{sgd:main}).

\paragraph{Scope and extensions.}
The separation is proved for two units, but several ingredients extend beyond width two: identical units stay identical under GF and SGD at any width~\citep{joudaki2026,simsek2021}, and for bias-free units and rotation-invariant inputs in any dimension, two gates disagree with probability $\delta/\pi$. In a wider network other units may fit the target, so a clone pair need not block learning; the open question is when the input mass on which units respond differently controls the risk. Explicit values of the burn-in $L$ and the step-size cutoff $\eta_0$ are left open; the simulations test the mechanism, not these constants. For smooth activations the wedge becomes a transition layer whose sampling is open.

In this model a longer run at fixed batch size and step size cannot replace a larger batch over horizons of order $e^{c/\eta}$ updates (Section~\ref{sec:checkpoint}), and perturbations that act on each unit separately restore disagreement mass (Appendix~\ref{app:controls}), as in shrink-and-perturb~\citep{ash2020}. L2 decay toward zero, which aids plasticity elsewhere~\citep{dohare2024}, here contracts the separation of near-clones, whereas regularization toward initialization~\citep{kumar2024} would pull them apart.

\paragraph{Use of AI tools.}
Generative AI tools were used as assistants, under the authors' direction, for literature search and checking of references; additional re-checking of derivations and editing of the written presentation of statements and proofs; feedback on the design of the simulations, their implementation, and plotting; and editing and restructuring of the text. The authors led and reviewed all such output and take full responsibility for the content of the paper, including text, claims and code produced with the aid of these tools.

\preprintappendix
\section{Finite-Block Tracking and Confinement}\label{app:tracking}

We use one concentration estimate for source entry, source confinement, the target
bridge, target confinement, and fresh success. Throughout the sample updates, take
$\relu'(0)=0$; zero-gate inputs have probability zero in the nonzero-hidden-vector
charts used below, so any fixed convention gives the same stochastic laws.
Every SGD update \eqref{sgd:update} can be written as
$\theta_{n+1}=\theta_n+\eta F(\theta_n)+\eta\xi_{n+1}$, where $F=-\nabla F_y$ is the
population field and $\xi_{n+1}$ is the centered batch-gradient error.  Individual
sample gradients need not be Lipschitz across gates; only $F$ must be.

\begin{lemma}[Tracking a smooth conditional mean]\label{sgd:tracking}
Let $\mathcal K\subset\R^6$ be compact, $0<a\le1$, and let $\mathcal K^{a}$ be its closed
$a$-neighborhood.  Suppose that on $\mathcal K^a$ the field $F$ is $L_F$-Lipschitz and
bounded by $M_F$, and that whenever $\theta_n\in\mathcal K^a$,
\[
 \theta_{n+1}=\theta_n+\eta F(\theta_n)+\eta\xi_{n+1},\qquad
 \E[\xi_{n+1}\mid\mathcal F_n]=0,\qquad\norm{\xi_{n+1}}\le M.
\]
Let $\ell>0$ and let $S_0$ be a set of initial states whose $F$-flows $\phi_t(\theta_0)$
remain in $\mathcal K$ for $t\in[0,\ell+1]$.  Put $N_\ell=\lceil\ell/\eta\rceil$.
There is $\bar\eta\in(0,1]$, depending only on $(a,\ell,L_F,M_F,M)$, such that for all
$\eta\le\bar\eta$ and $\theta_0\in S_0$,
\[
 \Pp\Bigl(\max_{n\le N_\ell}\norm{\theta_n-\phi_{n\eta}(\theta_0)}>a\Bigr)
 \le12\exp\Bigl[-\frac{a^2e^{-2L_F(\ell+1)}}{48M^2(\ell+1)\eta}\Bigr].
\]
In particular, on the complementary event the recursion stays in $\mathcal K^a$ up to
step $N_\ell$.
\end{lemma}
\begin{proof}
Let $\tau$ be the first $n$ with $\norm{\theta_n-\phi_{n\eta}(\theta_0)}>a$; it is a
stopping time, and $\theta_j\in\mathcal K^a$ for $j<\tau$.  For $m\le\tau\wedge N_\ell$,
put $\epsilon_m=\theta_m-\phi_{m\eta}(\theta_0)$.  Writing the flow in integral form,
\[
 \epsilon_m=\eta\sum_{j<m}\bigl[F(\theta_j)-F(\phi_{j\eta})\bigr]
 +\eta\sum_{j<m}\xi_{j+1}\ind_{j<\tau}
 -\sum_{j<m}\int_{j\eta}^{(j+1)\eta}\bigl[F(\phi_s)-F(\phi_{j\eta})\bigr]ds .
\]
The last sum is at most $C_E\eta$ with $C_E=(\ell+1)L_FM_F$.  Let
$Z=\max_{m\le N_\ell}\norm{\eta\sum_{j<m}\xi_{j+1}\ind_{j<\tau}}$.  Discrete Gronwall gives
$\norm{\epsilon_m}\le(Z+C_E\eta)e^{L_F(\ell+1)}$ for all $m\le\tau\wedge N_\ell$; the bound
at $m=\tau$ uses only states $\theta_j$, $j<\tau$, which lie in $\mathcal K^a$.  The
stopped increments $\eta\xi_{j+1}\ind_{j<\tau}$ form a bounded martingale difference
sequence.  Coordinatewise maximal Azuma and a union bound over six coordinates give
$\Pp(Z\ge z)\le12\exp[-z^2/(12M^2(\ell+1)\eta)]$.  Take $z=ae^{-L_F(\ell+1)}/2$ and
$\bar\eta$ with $C_E\bar\eta e^{L_F(\ell+1)}\le a/2$.  On $\{Z<z\}$ we get
$\norm{\epsilon_m}<a$ for every $m\le\tau\wedge N_\ell$, which forces $\tau>N_\ell$.
\end{proof}

The confinement arguments iterate Lemma~\ref{sgd:tracking} over blocks of fixed
physical length.  

\begin{lemma}[Block confinement]\label{sgd:block}
Let $\mathcal K$ be compact, let $\Psi$ be defined and $L_\Psi$-Lipschitz on
$\mathcal K^a$, let $\varrho>0$, and let $\mathcal I$ be a closed constraint set (in the
applications, an allocation interval or a projective constraint).  Suppose a block
duration $\ell_{\rm blk}$ and tolerance $a\in(0,1]$ with $aL_\Psi\le \varrho/4$ satisfy: for every
$x\in\mathcal K^a$ with $\Psi(x)\le \varrho$ and
$x\in\mathcal I$, the flow $\phi_t(x)$ stays in $\mathcal K$ with $\Psi\le3\varrho/2$ for
$t\in[0,\ell_{\rm blk}+1]$ and has $\Psi\le \varrho/2$ for $t\in[\ell_{\rm blk},\ell_{\rm blk}+1]$.  Let the hypotheses of
Lemma~\ref{sgd:tracking} hold on $\mathcal K$, and let $\zeta'$ be the first time the
recursion leaves $\mathcal I$.  If $\theta_0\in\mathcal K^a\cap\mathcal I$ and $\Psi(\theta_0)\le \varrho$, then
for $\eta\le\bar\eta$ and every integer $N\ge1$,
\[
 \Pp\bigl(\Psi(\theta_n)>2\varrho\text{ for some }n\le N\wedge\zeta'\bigr)
 \le\Bigl(1+\frac{N\eta}{\ell_{\rm blk}}\Bigr)\,
 12\exp\Bigl[-\frac{a^2e^{-2L_F(\ell_{\rm blk}+1)}}{48M^2(\ell_{\rm blk}+1)\eta}\Bigr].
\]
\end{lemma}
\begin{proof}
Split time into blocks of $n_{\rm blk}=\lceil\ell_{\rm blk}/\eta\rceil$ updates. Condition on
$\mathcal F_{kn_{\rm blk}}$ and the event that all preceding blocks satisfy their tracking
bounds and $\zeta'>kn_{\rm blk}$. The induction gives $\Psi(\theta_{kn_{\rm blk}})\le \varrho$ and
$\theta_{kn_{\rm blk}}\in\mathcal K^a\cap\mathcal I$.  Fresh samples make the
block a recursion of the same form, and Lemma~\ref{sgd:tracking} applies uniformly
over such starting points.  On its good event, the block stays within $a$ of the flow.
It therefore keeps $\Psi\le3\varrho/2+aL_\Psi<2\varrho$ and ends with $\Psi\le \varrho/2+aL_\Psi\le \varrho$,
so the next block starts in the same condition unless $\zeta'$ has occurred.  A union
bound over at most $1+N/n_{\rm blk}\le1+N\eta/\ell_{\rm blk}$ blocks proves the claim.
\end{proof}

\section{Source SGD: Projective Contraction}\label{app:sourceSGD}

The source minimum family has $w_i=a_ie_1$, $a_i>0$, and
$a_1^2+a_2^2=B_s:=1-2\lambda$.  On a compact positive-head neighborhood
introduce
\[
 u_i=w_i/a_i,\quad d=\norm{u_1-u_2},\quad
 \ell_a=\log(a_1/a_2).
\]
The proportional-neuron family is exactly $u_1=u_2$.  Unlike balance,
this family is invariant under every Euler sample update.

\subsection{A Local One-Step Estimate without an Additive Floor}\label{app:sourceonestep}

For a batch, write $R_j=f_\theta(X_j)-\relu(X_{j1})$ and $\varsigma=1-\eta\lambda$.
Positive homogeneity gives the exact map
\begin{equation}\label{sgd:sourceprojective}
 a_i^+=a_i\left[\varsigma-\frac\eta b\sum_jR_j\relu(u_i^\top X_j)\right],\qquad
 u_i^+=\frac{\varsigma u_i-(\eta/b)\sum_jR_j\relu'(u_i^\top X_j)X_j}
 {\varsigma-(\eta/b)\sum_jR_j\relu(u_i^\top X_j)}.
\end{equation}
All denominators exceed $1/2$ on the chosen compact set for small $\eta$.
Anchor the gates at unit 2, set
$v=b^{-1}\sum_jR_j\relu'(u_2^\top X_j)X_j$, and define
\[
 \Phi_v(u)=\frac{\varsigma u-\eta v}{\varsigma-\eta v^\top u}.
\]
This equals both actual projective maps whenever all batch gates agree.
When differentiating $\Phi_v$, hold the shared, actual residuals $R_j$ fixed.
For bounded $u,v$,
\[
 D\Phi_v(u)=I+\eta \mathcal H(v,u)+O(\eta^2),\qquad
 \mathcal H(v,u)=(v^\top u)I+uv^\top.
\]
Writing $e_s=u_1-u_2$, the matched separation satisfies
\begin{equation}\label{sgd:matchedsource}
 e_{s,\mathrm{mat}}^+
 =[I+\eta \mathcal H(v,u_2)+\mathcal E]e_s,\qquad
 \norm{\mathcal E}\le C\eta d+C\eta^2.
\end{equation}
At an exact source minimum, $u_i=e_1$ and
$R_j=-2\lambda\relu(X_{j1})$, so
\[
 \E v=-\lambda e_1,\qquad
 \E \mathcal H(v,e_1)=-\lambda(I+e_1e_1^\top).
\]
The averaged matrix is continuous in the parameters, uniformly over compact
allocation intervals.  Bounded inputs bound $\mathcal H$ and the Taylor errors.
In a sufficiently small fixed source neighborhood of radius $r_0$,
\eqref{sgd:matchedsource} therefore yields
\[
 \E[\norm{e_{s,\mathrm{mat}}^+}^2\mid\theta]
 \le(1-c\lambda\eta)d^2,
 \qquad
 \E[\norm{e_{s,\mathrm{mat}}^+}\mid\theta]
 \le(1-c'\lambda\eta)d.
\]

For one sample whose unit gates disagree, the geometric wedge bounds are
\[
 \Pp(\mathcal G\mid\theta)\le Cd,\qquad
 \abs{u_i^\top X}\le Cd\norm{X},\qquad
 \abs{X_1}\le C(r_0+d)\norm{X}.
\]
Thus the source residual on this wedge is at most $C(r_0+d)\norm{X}$.
Comparing both numerator and denominator in \eqref{sgd:sourceprojective}
with the matched map gives
\[
 \norm{e_s^+-e_{s,\mathrm{mat}}^+}
 \le\frac{C\eta(r_0+d)}b\sum_{j=1}^b\ind_{\mathcal G_j}.
\]
Its conditional expectation is at most $C\eta(r_0+d)d$, independently of
$b$.  The symmetric part of $\E \mathcal H(v,e_1)$ is at most $-\lambda I$, and the averaged matrix
changes by $O(r_0)$ over the neighborhood, so the $L^1$ bound above holds with
$c'\ge1-C(r_0+\eta)/\lambda$.  Since $d\le Cr_0$ in this neighborhood, the disagreement
correction is at most $C\eta r_0 d$, so one may take $\gamma=\lambda-C(r_0+\eta_0)$; reducing $r_0$ and then $\eta_0$ gives $\gamma\ge\lambda/2$ with
\begin{equation}\label{sgd:sourceL1}
 \E[d_{n+1}\mid\mathcal F_n]\le(1-\gamma\eta)d_n
\end{equation}
throughout the source neighborhood.  
The disagreement correction is estimated in $L^1$: its squared contribution
can be $O(\eta^2r_0^2d)$ and cannot in general be absorbed into $d^2$.

For every batch, including samples with disagreeing gates, the head ratio obeys
\begin{equation}\label{sgd:sourceallocation}
 \abs{\ell_a^+-\ell_a}
 =\abs{\log\frac{\varsigma-(\eta/b)\sum_jR_j\relu(u_1^\top X_j)}
 {\varsigma-(\eta/b)\sum_jR_j\relu(u_2^\top X_j)}}\le C\eta d.
\end{equation}

\subsection{Allocation and Normal Confinement}\label{app:sourceconfine}

Take nested compact allocation intervals $\mathcal A\Subset\mathcal A'$.  Around the source
minimum manifold over $\mathcal A'$, let $\xi=0$ be the normal coordinates.  The
positive normal Hessian in Lemma~\ref{lem:local} yields a smooth uniformly
positive matrix $P(\ell_a)$ and
\[
 W_s(\xi,\ell_a)=\xi^\top P(\ell_a)\xi,\qquad
 \dot W_s\le-\kappa W_s
\]
for the population flow in a small tube.  The derivative of $P$ contributes
only cubic terms because $\dot\ell_a=O(d)=O(\norm{\xi})$.
Choose inner and outer tubes $\sqrt{W_s}\le r$ and $\sqrt{W_s}<2r$
inside the neighborhood of \eqref{sgd:sourceL1}.  A fixed population block
duration $t_{\rm blk}$ maps the inner tube into $\sqrt{W_s}\le r/4$ and stays inside
the inner tube, provided its allocation remains in $\mathcal A'$ over the duration $t_{\rm blk}+1$.  Taking $r$ small
ensures this allocation condition for a reference block starting in $\mathcal A$.

Stop the source process at the first outer normal-tube exit or allocation
exit from $\mathcal A$, and call that time $\zeta_s$. Killing the distance
at this time only decreases it, so \eqref{sgd:sourceL1} implies
\begin{equation}\label{sgd:sourcekilled}
 \E[d_n\ind_{n<\zeta_s}]\le d_0(1-\gamma\eta)^n,
 \qquad
 \E\sum_{n<\zeta_s}\eta d_n\le d_0/\gamma.
\end{equation}
If the initial allocation has margin $m>0$ from the endpoints of $\mathcal A$,
\eqref{sgd:sourceallocation} and Markov's inequality give
\begin{equation}\label{sgd:sourceallocationexit}
 \Pp(\text{allocation exits at or before the normal tube})
 \le Cd_0/(\gamma m).
\end{equation}
For normal confinement, source gradients are bounded on the slightly larger
compact chart and the population field is Lipschitz there.  Apply
Lemma~\ref{sgd:block} with $\Psi=\sqrt{W_s}=\norm{P(\ell_a)^{1/2}\xi}$ (Lipschitz on the
compact chart because $P$ is smooth and uniformly positive), inner radius $r$, constraint set $\{\ell_a\in\mathcal A\}$, block duration
$\ell_{\rm blk}=t_{\rm blk}$, and a fixed tolerance $a\le r/(4L_\Psi)$ (so $\varrho=r$); the block property above supplies its
flow hypotheses.  With $N=\lfloor T/\eta\rfloor$ it gives
\begin{equation}\label{sgd:sourcenormalexit}
 \Pp(\text{normal exit by }T\text{ at or before allocation exit})
 \le C(1+T/\eta)e^{-c/\eta}.
\end{equation}

\subsection{Entry from the Fixed Gaussian Initialization Region}\label{app:entry}

Proposition~\ref{prop:source} and local attraction give uniform source
convergence on $K$ to an interior compact segment of source clones.  Choose the target entry tube of
Proposition~\ref{sgd:targettheorem} first.  Then choose the source outer tube
and allocation intervals so that every point in that outer tube satisfies
the target's entry conditions, including its fixed normal tolerance $\nu$
and projective tolerance $\epsilon_0$.  The estimates of Appendices~\ref{app:sourceonestep}--\ref{app:sourceconfine} remain valid on smaller tubes.

For any small fixed $\delta_{\rm burn}>0$, a single finite population burn-in duration
$L$ sends every $\theta_0\in K$ into the inner source tube with allocation
margin $m$ and $d\le\delta_{\rm burn}/2$.  The paths up to $L$ lie in a common compact
set with positive heads and radii.  Lemma~\ref{sgd:tracking} shows that actual
SGD enters with $d\le\delta_{\rm burn}$ and the required margins with probability
at least $1-Ce^{-c_{\rm burn}/\eta}$.  
Choose $\delta_{\rm burn}$ so that the right side of
\eqref{sgd:sourceallocationexit} is at most the allowance
$\varepsilon/2$, lengthening this fixed $L$ if necessary.

Let $E_T$ be the event of good burn-in entry and no source stop through the
source checkpoint at time $\lfloor T/\eta\rfloor\eta$.  Equations
\eqref{sgd:sourcekilled}--\eqref{sgd:sourcenormalexit} give, after changing
constants to absorb $L$ and time rounding,
\begin{equation}\label{sgd:sourceinterface}
 \Pp(E_T^c\mid\theta_0)
 \le\varepsilon/2+C(1+T/\eta)e^{-c/\eta},\qquad
 \E[d_T\ind_{E_T}\mid\theta_0]\le Ce^{-\gamma T}.
\end{equation}

\section{Target Confinement and the Bridge}\label{app:bridge}

This appendix proves the confinement bound \eqref{sgd:targetnormalexit} used in Section~\ref{app:targetSGD}, in the coordinates defined there.

On $e=0$, the auxiliary mean field holds $\rho$ and $u$ fixed and satisfies
\[
 \dot h_a=-\lambda h_a+(q-B/2)r,\qquad
 \dot r=-\lambda r+(q-B/2)h_a.
\]
Consequently
\begin{equation}\label{sgd:commondynamics}
 \dot k=-2\lambda k,\qquad
 \dot B=-2\lambda B+(q-B/2)\sqrt{S^2k^2+4B^2}.
\end{equation}
At $k=0,B=B_t$ the normal Jacobian is
$\operatorname{diag}(-2\lambda,-B_t)$, uniformly over
$\abs\rho\le\rho_2$.  In a sufficiently small common normal neighborhood,
for $W_t=k^2+(B-B_t)^2$ one has
\begin{equation}\label{sgd:forcednormal}
 \dot W_t\le-\kappa W_t+C_e\norm e^2.
\end{equation}
Indeed, the full mean field differs smoothly from
\eqref{sgd:commondynamics} by $O(\norm e)$ in $\dot B$; Young's inequality
absorbs the resulting $O(\sqrt{W_t}\norm e)$ term.  The identity for $\dot k$
remains exact.  The continuous counterpart of
\eqref{sgd:transversenegative} contracts $\norm e$ inside this chart.

Choose a small common radius $R_W$, then reduce $\epsilon$ so that the forcing
$C_e\epsilon^2/\kappa$ is at most $R_W^2/256$.  Then \eqref{sgd:forcednormal} gives
$W_t(t)\le\max\{W_t(0),R_W^2/256\}$ and
$W_t(t)\le e^{-\kappa t}W_t(0)+R_W^2/256$ along the mean flow while $\norm e\le\epsilon$.
Hence a fixed block of duration $\ell_{\rm blk}\ge\kappa^{-1}\log(16/3)$ keeps
$\sqrt{W_t}\le R_W/4$ from any start with $\sqrt{W_t}\le R_W/4$, ends with
$\sqrt{W_t}\le R_W/8$, and never leaves $\sqrt{W_t}<R_W/2$.  Equation
\eqref{sgd:targetallocation} in continuous time bounds allocation motion
over one block by $C(\ell_{\rm blk}+1)\epsilon$; choose it smaller than half the buffer
$\rho_2-\rho_1$.  

The source checkpoint is not at the target amplitude $B_t$.  At
$e=k=0,B=B_s$, \eqref{sgd:commondynamics} is
\[
 \dot B=B(B_t-B).
\]
This logistic path stays between $B_s$ and $B_t$ and enters the stationary
inner tube after a fixed time $\ell_0$.  Uniform continuity over allocations
and directions gives $\nu,\epsilon_0>0$ such that the full auxiliary mean
trajectories from the proposition's initial set stay in a compact bridge
tube and enter strictly inside that stationary inner tube.  The bridge has
positive lower bounds on $h_a,r$, lies below $B=2q$, and therefore
also satisfies the projective contraction estimate, with $C_U$ and $\epsilon$ chosen on a chart containing the bridge tube.

The bridge is a single application of Lemma~\ref{sgd:tracking} over the fixed
duration $\ell_0$, with a tolerance smaller than the margin by which the bridge mean
trajectories enter $\sqrt{W_t}\le R_W/4$; its failure probability is at most $Ce^{-c/\eta}$.  Afterwards apply
Lemma~\ref{sgd:block} with $\Psi=\sqrt{W_t}$, $\varrho=R_W/4$, constraint set
$\{\norm e\le\epsilon,\ \abs\rho\le\rho_1\}$, and a fixed tolerance $a\le R_W/(16L_\Psi)$.
The auxiliary mean is smooth by \eqref{sgd:auxmean}, and bounded inputs bound all
martingale increments.  Together they give \eqref{sgd:targetnormalexit}:
\begin{equation*}
 \Pp(\text{common-tube exit by }N\text{ at or before other stops})
 \le C(N+1)e^{-c/\eta}.
\end{equation*}
The bound covers the bridge block, since $1+N\eta/\ell_{\rm blk}\le C(N+1)$.

\section{Proof of the Two-Phase Theorem}\label{app:composition}

\begin{proof}[Proof of Theorem~\ref{sgd:main}]
Part (a) follows from Proposition~\ref{prop:target} and the uniform escape
estimate in Appendix~\ref{app:sourceescape}, applied to the compact set $K$ in
Appendix~\ref{app:explicitK}.
For part (b), the source allocation bound \eqref{sgd:sourceallocationexit} supplies the fixed allowance $\varepsilon$.

Choose a compact segment containing all source clone limits from $K$.
On clones, the source allocation $\ell_a$ and target allocation satisfy
$\rho=\tanh(\ell_a/2)$.  Select nested target allocation intervals containing
this segment with strict margins.  Construct the target entry neighborhood
first, then fit the source outer tube and allocation interval inside it,
as specified above.  On the resulting source compact set, smoothness of the
coordinate changes gives
\begin{equation}\label{sgd:coordinatecompare}
 \norm{e_{\rm target}}\le C\norm{w_1/a_1-w_2/a_2}=Cd.
\end{equation}
For example,
$z_1-(a_1/a_2)z_2=a_1(0,u_1-u_2)$; orthogonal projection and the compact
positive-head bounds give \eqref{sgd:coordinatecompare}.

On the good source event $E_T$, all target entry conditions hold.
Condition on that random checkpoint and apply
Proposition~\ref{sgd:targettheorem} to the independent target batches.
By \eqref{sgd:sourceinterface} and \eqref{sgd:coordinatecompare}, with $N=\max\{1,\lfloor H/\eta\rfloor\}$,
\begin{align*}
 \Pp(\tau_w(T)\le H\mid\theta_0)
 &\le\Pp(E_T^c\mid\theta_0)
   +\frac{Cb}{\eta}\E[\norm{e_T}\ind_{E_T}\mid\theta_0]
   +C(1+H/\eta)e^{-c/\eta}\\
 &\le\varepsilon/2+\frac{Cb}{\eta}e^{-\gamma T}
   +C\left(1+\frac{T+H}{\eta}\right)e^{-c_*/\eta},
\end{align*}
where $c_*$ is at most each of the burn-in, source and target confinement rates.
Keeping instead the finite-$N$ factor of
Proposition~\ref{sgd:targettheorem} gives, for every integer $N\ge1$,
\begin{equation}\label{sgd:finiteN}
 \Pp(\tau_w(T)\le N\eta\mid\theta_0)\le\varepsilon/2+Cb\,e^{-\gamma T}\min\{N,\eta^{-1}\}
 +C(1+T/\eta+N)e^{-c_*/\eta}.
\end{equation}
Applied to \eqref{sgd:finiteN} with $\varepsilon<p/4$ and $C(1+T/\eta+N)e^{-c_*/\eta}\le p/8$, the argument of Section~\ref{sec:checkpoint} gives the same necessary sizes for the two-phase SGD run, with $e^{\gamma T}$ in place of $e^{\lambda T}$ and a smaller $c_p$.

For fresh success, Proposition~\ref{prop:target} gives target population convergence on $K$
to risk $R_{\rm good}<0.3$.  For each point choose a finite time at which
risk is below a fixed intermediate level, say $0.3$.  Finite-time continuity
and a finite cover of $K$ give a common time horizon $t_0$ and a positive
risk margin for these crossings.  The relevant paths lie in a common compact
set with nonzero heads and radii.  Lemma~\ref{sgd:tracking}, with a tolerance
small compared with the risk margin, gives
\[
 \Pp(\tau_f>t_0\mid\theta_0)\le Ce^{-c_f/\eta}.
\]

A union bound proves \eqref{sgd:combined} after reducing $c_*$ to at most $c_f$ and enlarging constants.
\end{proof}

\begin{proof}[Proof of Corollary~\ref{sgd:fixedT}]
Fix $\theta_0\in K$ and $T$. The source gradient-flow path $\phi^s_t(\theta_0)$, $t\le T+1$, lies in
a compact set with positive heads and radii (Lemma~\ref{lem:invariant}). On a closed neighborhood
of it the population field is Lipschitz, including near coincident directions by the
kernel expansion of Appendix~\ref{app:populationequations}, and sample gradients are bounded
independently of $b$. Since $\phi^s_T(\theta_0)\in\calD$, Proposition~\ref{prop:target} shows that
target gradient flow from it converges to risk $R_{\rm good}<0.3$, so choose a finite $t_1$ after which its risk remains at most $0.3$. By continuous dependence on initial conditions over $[0,t_1+1]$ and
Lipschitz continuity of the risk, choose $a>0$ such that every target flow started within $2a$
of $\phi^s_T(\theta_0)$ stays in a compact set with positive heads and radii and has risk at most
$0.35$ on $[t_1,t_1+1]$. For $T>0$, Lemma~\ref{sgd:tracking} with $\ell=T$ and tolerance $a/2$ gives
$\norm{\theta_{\lfloor T/\eta\rfloor}-\phi^s_T(\theta_0)}\le a$ with probability at least $1-Ce^{-c/\eta}$; the
rounding $T-\lfloor T/\eta\rfloor\eta<\eta$ changes the flow by $O(\eta)$. When $T=0$ the initial checkpoint is exact and this source-tracking step is omitted. An application over the target
duration $t_1+1$, with a tolerance small compared with the risk margin, shows that target SGD
has risk at most $0.4$ at update $\lceil t_1/\eta\rceil$, whose time lies in $[t_1,t_1+\eta]$, with probability at least $1-Ce^{-c/\eta}$.
Replacing $t_1$ by $t_1+1$ gives the claim.
\end{proof}

\paragraph{Dependence on the weight decay.}\label{app:lambda}
The proofs use $\lambda$ only through the following facts: $\lambda>0$; $B_s=1-2\lambda>0$ (Lemma~\ref{lem:local}, Proposition~\ref{prop:source}); $A=q-\lambda>0$ (Propositions~\ref{prop:boundary} and~\ref{prop:target}, Lemma~\ref{lem:escape}); $B_s<B_t$, which holds because $2q>1$, and $B_t+\lambda/2=2q-\frac32\lambda<2q$ (Section~\ref{app:targetSGD}); and $R_{\rm good}=1-2q^2+2\lambda^2<0.3$, which gives \eqref{sgd:riskgap} and the intermediate risk levels $0.3$ and $0.35$ above. The level $R_{\rm col}$, the function $R_{\min}$ in \eqref{eq:anglegap}, the initial-risk bound $0.5928$, and $K$ and $\Pp(K)$ do not depend on $\lambda$. These facts hold for every $\lambda\in(0,1/5]$, where $R_{\rm good}\le0.27$; the binding condition is $R_{\rm good}<0.3$, and the others need only $\lambda<1/2$. Hence all results of Sections~\ref{sec:setup}--\ref{app:targetSGD} hold for every fixed $\lambda\in(0,1/5]$, with constants that depend on $\lambda$ and may degrade as $\lambda\to0$.

\section{Population Equations and Regularity}\label{app:populationequations}

We use the model and notation of Section~\ref{sec:setup}; the symbols $h_i$, $k_i$ and $s_i$ defined below are local to Appendices~\ref{app:populationequations}--\ref{app:sourceescape}. We write
$q=2/\pi$, $A=q-\lambda$, and $B_s=1-2\lambda$. The population objective is
the same for the disk and $N(0,I_2)$ by the radial equivalence in Section~\ref{sec:setup}. The ReLU kernel for unit directions separated by $t\in[0,\pi]$ is
\begin{equation*}
 J(t)=\frac{\sin t+(\pi-t)\cos t}{2\pi},\qquad
 g(t)=-J'(t)=\frac{(\pi-t)\sin t}{2\pi}.
\end{equation*}
In particular, $J(0)=1/2$, $J(\pi)=0$, $g>0$ on $(0,\pi)$, and
\begin{equation}\label{eq:expansion}
 J(t)=\frac12-\frac{t^2}4+\frac{t^3}{6\pi}+O(t^4),\qquad
 g(t)=\frac t2+O(t^2).
\end{equation}
For the source sector write $\alpha=\vartheta_1$, $\beta=-\vartheta_2$,
$\delta=\alpha+\beta$, and define
\[
 q_1=J(\alpha)-\omega_1/2-\omega_2J(\delta),\qquad
 q_2=J(\beta)-\omega_2/2-\omega_1J(\delta).
\]
With $h_i=a_i/r_i$, the exact source equations are
\begin{align}
 \dot a_i&=r_iq_i-\lambda a_i,&
 \dot r_i&=a_iq_i-\lambda r_i,&
 \dot \omega_i&=(a_i^2+r_i^2)q_i-2\lambda \omega_i,\label{eq:amplitudes}\\
 \dot\alpha&=h_1[-g(\alpha)+\omega_2g(\delta)],&
 \dot\beta&=h_2[-g(\beta)+\omega_1g(\delta)].\label{eq:sourceangles}
\end{align}
For the target, $\E[\norm{X}\relu(u^\top X)]=q$ for every unit $u$.
Consequently, \eqref{eq:amplitudes} holds with
\begin{equation}\label{eq:targetq}
 q_i=q-\omega_i/2-\omega_jJ(\delta),\qquad
 \dot\delta=(h_1\omega_2+h_2\omega_1)g(\delta),\quad j\ne i.
\end{equation}
Here $\delta\in(0,\pi)$ is the angular separation; the absolute orientation
does not enter the radial target risk.

For either task, set $k_i=a_i^2-r_i^2$ and $s_i=a_i+r_i$.  Direct calculation gives
\begin{equation}\label{eq:identities}
 k_i(t)=k_i(0)e^{-2\lambda t},\quad
 \dot h_i=q_i(1-h_i^2),\quad
 (\log h_i)^{\boldsymbol\cdot}=-\frac{k_i}{\omega_i}q_i,\quad
 (\log s_i)^{\boldsymbol\cdot}=q_i-\lambda.
\end{equation}

Whenever $a_i,r_i>0$ and $q_i\ge-\omega_i/2$, one has
\begin{equation}\label{eq:hbound}
 D^+\abs{\log h_i}\le\frac{\abs{k_i}}2,
 \qquad
 \abs{\log h_i(t)}\le
 \abs{\log h_i(0)}+\frac{\abs{k_i(0)}}{4\lambda}.
\end{equation}
Indeed, $\log h_i$ has the sign of $k_i$.  Thus $h_i$ stays bounded above and
away from zero.  If $q_i$ is bounded, the logarithmic equations for $a_i,r_i$
then preclude a zero at finite time.

On bounded parameter sets with $r_i>0$, the only potentially unbounded
population Hessian terms are bounded by constants times $a_i/r_i$.
This follows by differentiating the two-homogeneous Gaussian kernel:
its second derivatives in $w_i$ are of size $O(1/r_i)$, multiplied by $a_i$.
The kernel expansion also gives bounded second derivatives at coincident
directions.  Consequently, \eqref{eq:hbound} bounds the Hessian along the
trajectories considered below.  Coercivity bounds their parameters and
gradients.  Energy dissipation and uniform continuity therefore imply
\begin{equation}\label{eq:gradzero}
 \int_0^\infty\norm{\nabla F_s}^2\,dt<\infty,
 \qquad \norm{\nabla F_s(\theta(t))}\longrightarrow0.
\end{equation}
At a limiting zero unit, $k_i\to0$ forces both $a_i,r_i\to0$, and that unit's
gradient tends continuously to zero.  

\section{Global Source Convergence}\label{app:sourcepopulation}

\begin{lemma}[Invariance and angular caps]\label{lem:invariant}
Source trajectories from $\calD$ remain in $\calD$.  Every parameter is
bounded, every $a_i,r_i$ is positive at finite time, and
\[
 \alpha(t)\le\max\{\alpha(0),\pi/2\}<\pi,\qquad
 \beta(t)\le\max\{\beta(0),\pi/2\}<\pi.
\]
\end{lemma}
\begin{proof}
In the region, $J(\alpha),J(\beta)\ge J(\delta)$ and $\omega_j\le1$, so
$q_i\ge-\omega_i/2$.  Apply \eqref{eq:hbound}.  At $\omega_i=1$, $q_i\le0$, hence
$\dot \omega_i<0$.  At $\alpha=0$ or $\beta=0$, \eqref{eq:sourceangles} points
inward; more precisely $g(t)\le t/2$ gives
$\dot\alpha\ge-h_{\max}\alpha/2$ and the analogous bound for $\beta$,
so these angles stay strictly positive.  At $\delta=\pi$,
$\dot\delta=-h_1g(\alpha)-h_2g(\beta)<0$.
Hence no boundary of $\calD$, including the zero-unit boundaries, is reached in finite time.  Coercivity gives bounded parameters.  Finally $g$ decreases
on $[\pi/2,\pi]$.  If $\alpha\ge\pi/2$, then $\delta\ge\alpha$ and
$\dot\alpha\le-h_1(1-\omega_2)g(\alpha)\le0$; the other angle is identical.
\end{proof}

Define the fully active clone manifold
\begin{equation}\label{eq:sourceS}
 \calS=\{a_i=r_i=\sqrt{\omega_i},\ w_i=r_ie_1,\ \omega_i>0,\ \omega_1+\omega_2=B_s\}.
\end{equation}
Its closure adds two endpoints, each with one absent unit.

\begin{lemma}[Source minima and local attraction]\label{lem:local}
Every point of $\overline{\calS}$ is a global source minimizer with value
$\lambda-\lambda^2$.  The manifold $\calS$ is normally attracting.
On each interior compact segment its five normal decay rates are
$B_s,1/2,\lambda,2\lambda,2\lambda$.
\end{lemma}
\begin{proof}
Since $\norm{\relu(w_i^\top X)}_{L^2}=r_i/\sqrt2$,
the penalty is at least $\lambda\sum_i\abs{\omega_i}\ge\lambda\sqrt2\norm{f}_{L^2}$.
The reverse triangle inequality then gives
\[
 F_s\ge\frac12(\norm{f}_{L^2}-1/\sqrt2)^2
       +\lambda\sqrt2\norm{f}_{L^2}\ge\lambda-\lambda^2.
\]
Equality holds at \eqref{eq:sourceS}, where $R_s=\lambda^2$.

For the local statement write $p=\omega_1/B$, $\tilde B=B-B_s$,
$\mu=p\vartheta_1+(1-p)\vartheta_2$, and use $k_1,k_2$ as balance coordinates.
The exact objective and \eqref{eq:expansion} give, uniformly for interior $p$,
\begin{equation*}
 F_s-(\lambda-\lambda^2)=
 \frac{\tilde B^2}{4}+\frac{B_s\mu^2}{4}
 +\frac{\lambda B_sp(1-p)\delta^2}{2}
 +\lambda\sum_i\frac{k_i^2}{8\omega_i}+o(\norm{\chi}^2),
\end{equation*}
where $\chi=(\tilde B,\mu,\delta,k_1,k_2)$.  The only null direction is tangent to
$\calS$.  Linearizing \eqref{eq:amplitudes}--\eqref{eq:identities} yields the
stated rates.  In a small tube the objective gap is
comparable to $\norm{\chi}^2$, its squared gradient dominates that gap, and its
gradient is at most $C\norm{\chi}$.  Energy dissipation gives exponential normal
decay, while the remaining path length is exponentially small.  A smaller
tube is thus trapped and every trajectory in it converges exponentially to
one point of $\calS$.  These estimates are uniform on interior compact segments.
\end{proof}

\begin{proposition}[Global, fully active source limit]\label{prop:source}
Every source trajectory from $\calD$ converges to a point of $\calS$.
\end{proposition}
\begin{proof}
We first classify the omega-limit points.  By
\eqref{eq:gradzero} and $k_i\to0$, they are balanced critical points, allowing
absent units.  At a fully active critical point with $\alpha,\beta>0$ and
$\delta<\pi$, the angular equations imply
\[
 \omega_2=\frac{g(\alpha)}{g(\delta)},\qquad
 \omega_1=\frac{g(\beta)}{g(\delta)}.
\]
Substitution in $q_1=\lambda$ gives the contradiction
\begin{equation*}
 \lambda=-\frac{\beta\sin\alpha\sin\delta+\alpha(2\pi-\delta)\sin\beta}
                  {2\pi(\pi-\delta)\sin\delta}<0.
\end{equation*}
A zero source angle with the other positive is excluded by its inward angular
derivative; $\delta=\pi$ is excluded by its strictly negative derivative and
the angular caps.  Hence every fully active critical point lies in $\calS$.
A single active unit at a critical point must point along $e_1$ and have
amplitude $B_s$: its angular derivative vanishes only at $0$ or $\pi$, and
$\pi$ is excluded.  The only remaining critical point is the origin.

The omega-limit set is nonempty, compact, connected, and has a single objective
value.  The origin has value $1/4$, distinct from $\lambda-\lambda^2$.
If the origin belonged to this set, the trajectory would therefore converge
to the origin.  Then $\omega_1,\omega_2\to0$; the positive lower bounds on $h_i$ and
the angular caps in Lemma~\ref{lem:invariant} force $\alpha,\beta\to0$ in
\eqref{eq:sourceangles}.  Thus $q_i\to1/2>\lambda$, and
$\dot s_i=(q_i-\lambda)s_i$ eventually grows at a fixed exponential rate,
a contradiction.

The omega-limit set is consequently contained in $\overline{\calS}$.
If it meets $\calS$, Lemma~\ref{lem:local} implies convergence to an interior
point.  Otherwise connectedness forces convergence to one of the two endpoints.
Suppose $\omega_1\to B_s$, $\alpha\to0$, and $\omega_2\to0$.  The angular cap gives
\[
 q_2=2\lambda J(\beta)+o(1),
\]
which is eventually bounded below by a positive constant, uniformly over the allowed weak-unit angles.  Hence
$\dot h_2=q_2(1-h_2^2)$ yields exponential convergence $h_2\to1$.
Also $h_1\to1$ exponentially by $k_1\to0$ and $\omega_1\to B_s>0$.
The weak angle equation is $\dot\beta=-2\lambda g(\beta)+o(1)$;
the cap and the sign away from zero imply $\beta\to0$.
Now \eqref{eq:expansion} gives
$\dot\delta=[-\lambda+o(1)]\delta$, so $\delta$ decays at least
as $e^{-\lambda t/2}$ after a sufficiently large time.  But
\begin{equation*}
 \frac{d}{dt}\log\frac{s_1}{s_2}
 =q_1-q_2
 =J(\alpha)-J(\beta)+(\omega_2-\omega_1)[1/2-J(\delta)]
 =O(\delta^2).
\end{equation*}
This is integrable, so $s_1/s_2$ has a finite positive limit, contradicting
the supposed endpoint.  The other endpoint is identical.
\end{proof}

\section{Global Target Convergence and the Risk Gap}\label{app:targetpopulation}

\begin{proposition}[Fresh target success]\label{prop:target}
Suppose initially $a_i,r_i>0$, $\omega_i\le2q$ and $0<\delta<\pi$.
Target gradient flow keeps both units positive, has $\delta\to\pi$ and
$\omega_1,\omega_2\to2A$, and satisfies $R_t\to R_{\rm good}$.
In particular it succeeds from every point of $\calD$.
\end{proposition}
\begin{proof}
At $\omega_i=2q$, \eqref{eq:targetq} gives $q_i\le0$, so this amplitude bound is
invariant.  It also gives $q_i\ge-\omega_i/2$, and \eqref{eq:hbound} again bounds
$h_i$ above and away from zero and excludes finite-time zero units.
Since $q_i\ge q-B/2$, whenever $B\le2A$ one has $q_i\ge\lambda>0$ and
\begin{equation}\label{eq:Bbarrier}
 \dot B\ge2\sum_i \omega_i(q_i-\lambda)\ge B(2A-B).
\end{equation}
Thus $B\ge\min\{B(0),2A\}>0$.  The angle equation implies
\[
 \dot\delta\ge h_{\min}\min\{B(0),2A\}g(\delta),
\]
forcing $\delta\to\pi$.
Once $J(\delta)\le A/(8q)$, if $\omega_i\le A$ then
$q_i-\lambda\ge A/4$, and \eqref{eq:amplitudes} gives
$\dot \omega_i\ge(A/2)\omega_i$.  Both amplitudes therefore reach $A$ in finite time
and stay above it.  With $k_i\to0$, their equations become
\[
 \dot \omega_i=2\omega_i(A-\omega_i/2)+o(1).
\]
The sign above and below the unique positive zero yields $\omega_i\to2A$.
The limiting prediction has the form $2A\abs{u^\top X}$, and direct
evaluation gives $R_t\to1-2q^2+2\lambda^2=R_{\rm good}$.
\end{proof}

A positively proportional pair is a single ReLU $\omega\relu(u^\top x)$.  Its radial target risk is
\begin{equation}\label{eq:clonegap}
 R_t=1+\omega^2/4-q\,\omega\ge1-q^2=0.5947152\ldots>0.4.
\end{equation}
More generally, at fixed separation $0\le\delta\le\pi$, minimizing the
quadratic risk over the two amplitudes gives
\begin{equation}\label{eq:anglegap}
 R_t\ge R_{\min}(\delta):=1-\frac{q^2}{1/2+J(\delta)}.
\end{equation}
For $\delta>0$ this follows by inverting the $2\times2$ Gram matrix with
diagonal $1/2$ and off-diagonal $J(\delta)$; the $\delta=0$ value follows
by continuity.

\section{Source Contraction and Target Escape Time}\label{app:sourceescape}

Fix a compact set $K\subset\calD$; in Theorem~\ref{sgd:main} it is the set of Appendix~\ref{app:explicitK}.  Proposition~\ref{prop:source} and local
attraction imply uniform eventual entry into neighborhoods of interior
clone points.  Indeed, each initial condition has a finite
entry time into a smaller attracting tube; finite-time continuity makes
the same true in an initial neighborhood.  A finite cover of $K$ provides
a common entry-time bound and an interior compact segment containing all
limiting clones.

\begin{lemma}[A nonzero source splitting coefficient]\label{lem:slow}
Uniformly for $\theta_0\in K$, the source separation satisfies
\begin{equation*}
 \delta_s(T)=C_s(\theta_0)e^{-\lambda T}(1+o(1)),\qquad
 0<C_-\le C_s(\theta_0)\le C_+<\infty.
\end{equation*}
The source amplitudes converge to positive limits uniformly bounded away
from zero, and $k_i(T)=k_i(0)e^{-2\lambda T}$.
\end{lemma}
\begin{proof}
The last two statements follow from the preceding finite-cover argument
and \eqref{eq:identities}.  Near an interior clone, $h_i-1=O(e^{-2\lambda T})$.
Using $\alpha,\beta\ge0$ and $\alpha+\beta=\delta$, expand
\eqref{eq:sourceangles} to obtain
\[
 \frac{d}{dT}\log\delta_s(T)
 =-\lambda+O\bigl(\delta_s+\abs{B-B_s}+\abs{h_1-1}+\abs{h_2-1}\bigr).
\]
The error is exponentially integrable by Lemma~\ref{lem:local}.
The separation is strictly positive at every finite time, so integration
gives a finite strictly positive limit of $e^{\lambda T}\delta_s(T)$.
Finite-time continuity and the uniform tail bound make that limit continuous
on $K$, proving the uniform positive bounds.
\end{proof}

\begin{lemma}[Uniform target escape time]\label{lem:escape}
For sufficiently small fixed $\delta_\star>0$, let $t_{\delta_\star}(T)$ be the first target
time, after source duration $T$, at which $\delta=\delta_\star$.
Uniformly on $K$ and for $T\ge T_\star$, with $T_\star$ large enough that $\delta_s(T_\star)<\delta_\star$,
\begin{equation}\label{eq:escape}
 t_{\delta_\star}(T)=\frac\lambda A T+O(1).
\end{equation}
At that time both amplitudes belong to a fixed compact positive interval.
\end{lemma}
\begin{proof}
All bounds below are uniform over $K$.  Choose $T$ large.  Initially the
target amplitudes are close to an interior source clone, so their total
has a positive lower bound and their ratio lies in $[R_0^{-1},R_0]$ for
some fixed $R_0\ge1$.  Target invariance and \eqref{eq:Bbarrier} give
\begin{equation*}
 0<B_{\min}\le B(t)\le4q,
 \qquad \abs{k_i(t)}\le C_0e^{-2\lambda(T+t)}.
\end{equation*}
Stop before $\delta=\delta_\star$ if the amplitude ratio first leaves
$[(2R_0)^{-1},2R_0]$.  Within this stopped tube,
$\omega_i\ge m:=B_{\min} /(1+2R_0)>0$.  Hence
\[
 h_i=1+O(\abs{k_i}),\qquad
 a_i^2+r_i^2=\sqrt{4\omega_i^2+k_i^2}=2\omega_i+O(k_i^2).
\]
By taking $T$ large and $\delta_\star$ small, \eqref{eq:targetq} implies
$\dot\delta\ge v\delta$ for a fixed $v>0$.  Before the stopping time,
\begin{equation}\label{eq:integrals}
 \int\delta(t)^j\,dt\le\frac{\delta_\star^j}{jv}\quad(j=1,2),
 \qquad \int\abs{k_i(t)}^j\,dt=O(e^{-2j\lambda T}).
\end{equation}
Put $D(\delta)=1/2-J(\delta)=O(\delta^2)$.  The exact amplitude equations give
\begin{equation*}
 \frac{d}{dt}\log\frac{\omega_1}{\omega_2}
 =-2(\omega_1-\omega_2)D(\delta)+O(k_1^2+k_2^2).
\end{equation*}
The total change of this logarithmic ratio is therefore at most
$C\delta_\star^2+Ce^{-4\lambda T}$.  First choose $\delta_\star$ small and then $T$ large
so that this is less than $\log2$.  The ratio boundary cannot be reached.
Thus the compact positive-amplitude tube persists until $\delta=\delta_\star$;
the lower angular growth bound ensures this occurs in finite time.

In this tube, summing the amplitude equations gives
\begin{equation*}
 \dot B=B(2A-B)+E(t),\qquad
 \abs{E(t)}\le C\bigl(\delta^2+k_1^2+k_2^2\bigr).
\end{equation*}
Writing $z=B-2A$, one has $\dot z=-Bz+E$.  Variation of constants and
$B\ge B_{\min}$ imply
\begin{equation}\label{eq:Bint}
 \int_0^{t_{\delta_\star}}\abs{B-2A}\,dt
 \le \frac{\abs{B(0)-2A}+\int_0^{t_{\delta_\star}}\abs{E(t)}\,dt}{B_{\min}}=O(1).
\end{equation}
Finally \eqref{eq:expansion} and \eqref{eq:targetq} yield
\[
 \frac{d}{dt}\log\delta=\frac B2+
 O(\delta+\abs{k_1}+\abs{k_2}).
\]
Integrating, using \eqref{eq:integrals} and \eqref{eq:Bint}, gives
\[
 \log\frac{\delta_\star}{\delta_s(T)}=A t_{\delta_\star}(T)+O(1).
\]
Lemma~\ref{lem:slow} now proves \eqref{eq:escape}.
\end{proof}

\begin{proof}[Proof of the hitting-time law]
Choose $\delta_\star$ as in Lemma~\ref{lem:escape}, further reducing it so that
$R_{\min}(\delta_\star)>0.4$.  Since $J$ decreases, \eqref{eq:anglegap} precludes
crossing the risk threshold before $t_{\delta_\star}(T)$.  At the exit time, the states
have separation $\delta_\star$, $m\le \omega_i\le2q$, and uniformly small balance errors.
Including the irrelevant overall direction, they lie in a fixed compact
subset of the target basin in Proposition~\ref{prop:target}.
Every point of that compact set reaches risk strictly below $0.4$ in finite
time.  For each point, continuous dependence gives the same strict crossing
at its chosen time for a neighborhood of initial states.  A finite cover
therefore bounds the additional crossing time uniformly.  Hence
$t_{\delta_\star}(T)\le\tau_{\rm GF}(T)\le t_{\delta_\star}(T)+O(1)$ for large $T$. For $T$ in a bounded interval $[0,T_*]$, the checkpoints $\phi^s_T(\theta_0)$ with $\theta_0\in K$ form a compact subset of $\calD$ (Lemma~\ref{lem:invariant}), and the same finite-cover argument bounds $\tau_{\rm GF}(T)$ uniformly there. This proves \eqref{sgd:gfcontrast}.
The same finite-cover argument applied directly to $K$ supplies a uniform
fresh-training time.
\end{proof}

\section{The Compact Gaussian Basin and Initial Risk}\label{app:explicitK}

For one unit under independent scalar standard Gaussians, $a$ is independent
of the Rayleigh radius $r$ and its uniform direction.  For $x\ge0$,
\[
 \Pp(a>0,ar>x)
 =\int_0^\infty\frac{e^{-a^2/2}}{\sqrt{2\pi}}
                 e^{-x^2/(2a^2)}\,da=\frac12e^{-x}.
\]
Thus, conditional on a positive head, $\omega=ar$ has distribution
$\operatorname{Exp}(1)$, and
$\Pp(a>0,ar<1)=(1-e^{-1})/2$.
The two ordered angular triangles
$\alpha,\beta>0$, $\alpha+\beta<\pi$ have combined area $\pi^2$
out of the direction square of area $4\pi^2$.  Their probability is $1/4$,
independent of the amplitudes.  This proves
\[
 \Pp(\calD)=\frac{(1-e^{-1})^2}{16}.
\]
Conditional on two positive heads, $B$ has distribution $\operatorname{Gamma}(2,1)$.
The condition $B<3/4$ already ensures $\omega_i<1$, so
\[
 \Pp(\calD_0)=\frac1{16}\Pp(\operatorname{Gamma}(2,1)<3/4)
 =\frac{1-(7/4)e^{-3/4}}{16}>0.01.
\]
For any point in $\calD_0$, positivity of the kernel gives
\[
 R_t=1+\frac{\omega_1^2+\omega_2^2}{4}+\omega_1\omega_2J(\delta)-qB
 \ge1+\frac{B^2}{8}-qB
 \ge1+\frac9{128}-\frac3{2\pi}>0.5928,
\]
since the last quadratic decreases on $[0,3/4]$.

\paragraph{An explicit compact subset.}
Put $m=10^{-4}$ and $M=5$, and define $K$ as the union of the two neuron orderings satisfying
\[
 m\le a_i\le M,\quad m\le\|w_i\|\le M,\quad
 \omega_1+\omega_2\le\tfrac34-m,\quad
 \vartheta_1\ge m,\quad\vartheta_2\le-m,\quad
 \vartheta_1-\vartheta_2\le\pi-m.
\]
This is a compact subset of $\calD_0$, bounded away from zero radii and angular branch cuts.

Under independent standard Gaussian parameters, the probability removed from $\mathcal D_0$ is bounded by
\begin{equation*}
 \Pp(\mathcal D_0\setminus K)\le
 \frac{2m}{\sqrt{2\pi}}+m^2+
 \frac{2e^{-M^2/2}}{M\sqrt{2\pi}}+2e^{-M^2/2}
 +\frac{3m}{\pi}+\frac{m}{4e}<1.93\times10^{-4}.
\end{equation*}
Indeed, the first four terms bound the two small positive heads, two small Rayleigh radii, two large heads (Mills' bound), and two large radii. Each angular strip $|\vartheta_i|<m$ has probability $m/\pi$. The density of the absolute difference of two uniform angles is $(2\pi-d)/(2\pi^2)$ for $d\in(0,2\pi)$, so the strip $\pi-m<d<\pi$ has probability at most $m/\pi$. Finally, conditional on two positive heads the sum $B$ has density $xe^{-x}\le e^{-1}$; the head-sign probability is $1/4$, giving $m/(4e)$ for the strip $3/4-m<B<3/4$. A union bound gives
\[
 \Pp(K)\ge\frac{1-(7/4)e^{-3/4}}{16}-1.93\times10^{-4}
 >0.01064>0.01.
\]

\section{Experimental Details}\label{app:experiments}

\paragraph{Implementation and random streams.}
The settings are in Section~\ref{sec:experiments}. C++17 code runs on one CPU core in double precision. Disagreement counts record sampled inputs on which the current gates differ. The target SGD stream depends only on the initialization and $T$, so checkpoints from the two source optimizers meet the same target samples; constructed checkpoints use the same streams in every cell. The code, with its seeds and the saved Gaussian initial vectors, is provided as ancillary files with this preprint.

\paragraph{Identical checkpoints (Figure~\ref{fig:paired}).}
The 64 saved vectors are Gaussian rejection draws from $\calD_0$; membership in the slightly smaller set $K$ is not enforced. Each source optimizer supplies checkpoints at $T\in\{0,100,200,300,400\}$ along one trajectory, with the two choices coinciding at $T=0$. Target risk is checked every ten updates through $H=200$. Table~\ref{tab:paired} gives the recovery counts; the source-GD columns are the protocol of Corollary~\ref{sgd:samecheckpoint}.
The dashed line in Figure~\ref{fig:paired} is $A^{-1}\log(1/\delta_T)+c$, with $c\approx3.9$ the median GD offset. Points well above the line come from checkpoints with a nearly inactive unit, such as the seven of Section~\ref{sec:experiments}; GD must regrow its amplitude, and the hitting time stays close to the GF delay $\lambda T/A$. For eight saved initializations at $T=400$, halving the GD step in both phases from $10^{-3}$ to $5\times10^{-4}$, then to $2.5\times10^{-4}$, changes target hitting times by at most $0.006$ and $0.003$.

\paragraph{Constructed checkpoints (Figure~\ref{fig:resolution}a).}
Balanced checkpoints at the source amplitude $B_s=0.9$:
$a_i=\norm{w_i}=\sqrt{0.45}$ and $w_i=\sqrt{0.45}\,(\cos(\delta_0/2),(-1)^{i+1}\sin(\delta_0/2))$.
The parameter grid in Section~\ref{sec:experiments} varies $\delta_0$, $b$ and $\eta$ independently; risk is checked every update. Wilson intervals use the independent streams within each cell. The hollow diamonds (checkpoints from Gaussian initializations) pool two source optimizers and four positive source durations from the same 64 draws; these 512 dependent records give fractions in bins of width $0.5$ in $\log_{10}\budget$ with at least five records, without binomial intervals. For the 12 cells with recovery fraction $\hat p\in(0,1)$, the effective-intensity ratio $\kappa_{\rm eff}=-\pi\lambda\log(1-\hat p)/\budget$ ranges from $0.49$ to $1.49$, with ten values below one. Values below one are expected: a sampled disagreement need not trigger recovery, and the amplitude transient from $0.9$ speeds early contraction. The two values above one come from cells with $5/64$ and $9/64$ recoveries, whose Wilson intervals contain the reference. In every cell, recovery occurs only after a sampled disagreement; among runs that sample one, the recovered fraction ranges from $0.81$ to $1$.

\paragraph{Long horizons (Figure~\ref{fig:resolution}b).}
Constructed checkpoints as above, $\eta=0.01$, 128 new streams per setting, risk checked every ten updates. The settings repeat cells of Figure~\ref{fig:resolution}(a), and the estimates agree within sampling error (for $b=64$, $\delta_0=10^{-4}$: $60/64$ there, $126/128$ here). GD reaches risk $0.4$ at $t=22.1$, $18.2$ and $14.3$ for $\delta_0=10^{-5},10^{-4},10^{-3}$. Figure~\ref{fig:resolution}(b) shows the settings with $b=16$, $\delta_0\in\{10^{-5},10^{-4}\}$ and $b=64$, $\delta_0=10^{-5}$. Table~\ref{tab:longhorizon} lists all five settings.
\begin{table}[h]
\caption{\textbf{Long horizons.} Recovered fractions by target time $t$ in the five constructed settings ($\eta=0.01$, 128 streams each), the Poisson reference, and the frozen-angle reference $1-e^{-\budget t/\pi}$ at $t=200$.}\label{tab:longhorizon}
\centering\small
\begin{tabular}{rrrccccc}
\toprule
$b$ & $\delta_0$ & $\budget$ & rec.\ $t\le60$ & rec.\ $t\le200$ & rec.\ $t\le450$ & $1-e^{-\budget/(\pi\lambda)}$ & frozen, $t=200$\\
\midrule
16 & $10^{-5}$ & 0.016 & 0.070 & 0.078 & 0.078 & 0.097 & 0.639\\
64 & $10^{-5}$ & 0.064 & 0.227 & 0.250 & 0.250 & 0.335 & 0.983\\
1 & $10^{-3}$ & 0.100 & 0.305 & 0.336 & 0.336 & 0.471 & 0.998\\
16 & $10^{-4}$ & 0.160 & 0.578 & 0.633 & 0.633 & 0.639 & 1.000\\
64 & $10^{-4}$ & 0.640 & 0.969 & 0.984 & 0.984 & 0.983 & 1.000\\
\bottomrule
\end{tabular}
\end{table}
The last recovery in the five settings occurs at $t=76.1$, $68.1$, $103.7$, $112.1$ and $64.2$, respectively. The angle of unrecovered runs decays at rate $\approx\lambda$, for example from $10^{-4}$ to $5.8\times10^{-7}$ by $t=100$, and reaches double-precision resolution (${\approx}10^{-15}$) at $t\approx460$ for $\delta_0=10^{-5}$ and $t\approx500$--$550$ for the larger $\delta_0$, long after the last recovery. Six of the 348 unrecovered runs sampled a disagreement. A rerun of all 640 runs in 80-bit extended precision through $t=1000$ gives the same recovery and first-disagreement times; only the floor of unrecovered angles moves, to about $10^{-19}$.

\paragraph{Trajectory (Figure~\ref{fig:mechanism}b--d).}
Panels (c,d) show four target SGD streams, which coincide to line width, and (c) shares the vertical scale of (b). GD crosses risk $0.4$ at $t\approx34.4$. The checkpoint has angle $1.67\times10^{-8}$. Fitted log-angle slopes are $-0.0500$ during source training and $0.5861$ and $-0.0500$ for target GD and SGD, matching $-\lambda$, $2/\pi-\lambda\approx0.5866$ and $-\lambda$. Panel (a) is a schematic with $\delta=0.3$.

\paragraph{Joint dependence on $T$ and $\eta$.}\label{app:joint}
SGD in both phases from the 64 saved initializations, $b=16$, $\eta\in\{4,2,1,0.5,0.25\}\times10^{-3}$, checkpoints along one source trajectory per initialization and step size at $T=0,25,\dots,400$, and a fresh target stream from each checkpoint through $H=200$, with risk checked every ten updates (Figure~\ref{fig:joint}). Records at different $T$ share a source trajectory and are dependent. Fresh runs ($T=0$) always recover, and at each fixed $T$ the recovered fraction increases as $\eta$ decreases, up to sampling fluctuations. The source time $T_{1/2}(\eta)$ at which half of the runs recover, interpolated linearly between adjacent grid points, is $241.7$, $257.9$, $265.6$, $284.4$ and $295.0$; subtracting $\lambda^{-1}\log(b/\eta)$, with $\lambda$ the source-GF angle rate rather than the $\gamma$ of Theorem~\ref{sgd:main}, leaves $75.8$, $78.2$, $72.0$, $76.9$ and $73.7$. The median checkpoint angle decays at rate $0.050$ for every $\eta$, and the recovered fractions for all step sizes fall close to one curve in $\budget_T=b\delta_T/\eta$. For every $\eta$, the runs that recover at $T=400$ include the near-opposite initializations of Section~\ref{sec:experiments}.
\begin{figure}[!htb]
\centering\includegraphics[width=.88\textwidth]{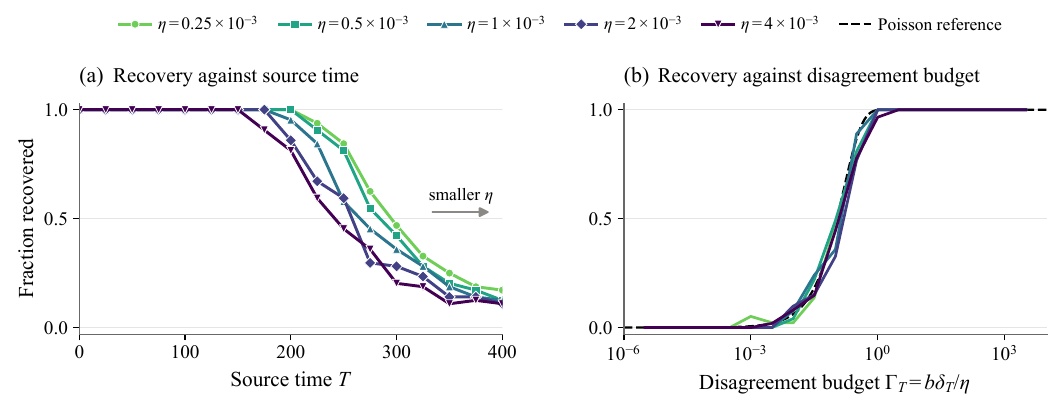}
\caption{\textbf{Joint dependence on source time and step size.} (a) Warm recovery declines with $T$, later for smaller $\eta$ (light to dark: increasing $\eta$). (b) Binned by $\budget_T=b\delta_T/\eta$, all step sizes fall near the Poisson reference (dashed); bins of width $0.5$ in $\log_{10}\budget_T$ with at least 8 records.}
\label{fig:joint}
\end{figure}

\section{Weight-Decay and Perturbation Controls}\label{app:controls}

\begin{figure}[t]
\centering\includegraphics[width=.88\textwidth]{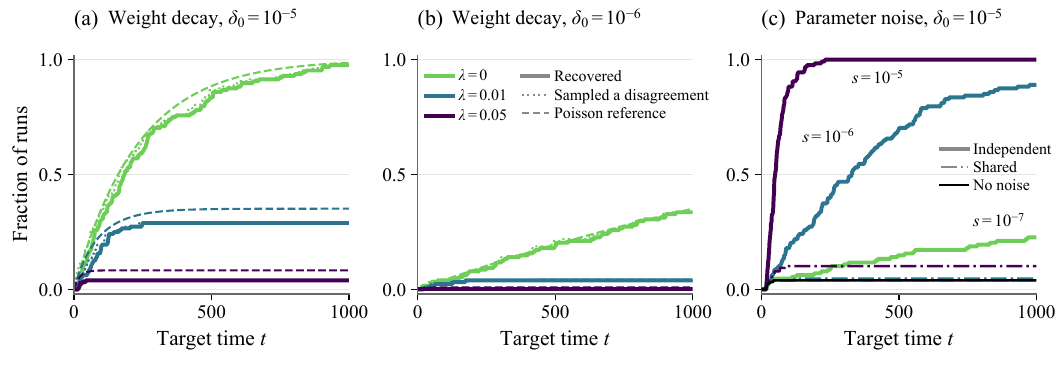}
\caption{\textbf{Weight-decay and parameter-noise controls} from constructed checkpoints ($b=16$, $\eta=0.01$, 128 runs per setting). (a,b) With target weight decay, recovery (solid) plateaus at or below the Poisson reference (dashed); without it, it keeps growing. Dotted: fraction that has sampled a disagreement. (c) Independent parameter noise raises recovery with $s$ (solid); shared noise barely does (dash-dotted); $\lambda=0.05$, $\delta_0=10^{-5}$; black: no noise.}
\label{fig:controls}
\end{figure}

Every run in Figure~\ref{fig:controls} starts from the balanced checkpoint of Appendix~\ref{app:experiments} with amplitude $0.9$, the source optimum for $\lambda=0.05$, and uses the same input streams in every setting; only the target phase changes. Disagreements are checked every update and risk every ten updates ($0.1$ physical-time units). Weight-decay runs stop at recovery; noise runs continue to $H=1000$.

\paragraph{Weight decay.}
The common-gate linearization at a balanced target clone contracts the transverse splitting mode at rate $\lambda$ (Proposition~\ref{prop:boundary}). After the amplitude transient toward $B_t=2(2/\pi-\lambda)$ we therefore use $\delta(t)\approx0.85\,\delta_0e^{-\lambda t}$, where $0.85$ approximates the norm-adjustment factors $\sqrt{0.9/B_t}\in[0.84,0.88]$. The cumulative Poisson intensity of this heuristic reference is
\[
 \Lambda(t)=\frac{0.85\,b\delta_0}{\pi\eta}\cdot\frac{1-e^{-\lambda t}}{\lambda};
\]
the reference curve is $1-e^{-\Lambda(t)}$, with $\Lambda(t)=0.85b\delta_0t/(\pi\eta)$ at $\lambda=0$. For $\delta_0=10^{-5}$, the observed recovery fractions plateau at $0.29$ and $0.04$ for $\lambda=0.01$ and $0.05$ (predicted $0.35$ and $0.08$), with no first disagreement observed after $t\approx240$ through $H=1000$. These input streams are independent of those in Appendix~\ref{app:experiments}, where the same setting gives $0.078$ by $t=450$ against a reference without the factor $0.85$ ($0.097$). For $\lambda=0$, after the amplitude transient the angle of unrecovered runs decays only slowly, from about $0.82\delta_0$ at $t=100$ to about $0.70\delta_0$ at $H=1000$; the cumulative disagreement intensity implied by these angles keeps growing nearly linearly, as in the reference $\Lambda(t)$, and the recovery fraction reaches $0.98$ and $0.34$ at $H=1000$ for $\delta_0=10^{-5}$ and $10^{-6}$, close to the reference values $0.99$ and $0.35$. In these runs recovery always follows a sampled disagreement. Population GD reaches risk $0.4$ by $t\approx26$ in every setting.

\paragraph{Parameter noise.}
After every SGD step, from the checkpoint with $\delta_0=10^{-5}$ and target $\lambda=0.05$, the full parameter vector of unit $i$ receives $s\sqrt\eta\,\xi_i$, with $\xi_i\sim N(0,I_3)$ fresh each step and independent of the input batches. The two units receive independent vectors or the same vector. Before escape, a local approximation models the transverse difference under independent noise as an Ornstein--Uhlenbeck process with stationary standard deviation $s/\sqrt\lambda$, so the angle fluctuates on the scale $s/(r\sqrt\lambda)\approx5.8s$ for hidden norm $r\approx0.77$; median angles among unrecovered runs are of this order. In this approximation the disagreement probability per step no longer decays, so the cumulative disagreement intensity grows roughly linearly in time. Recovery by $H=1000$ rises from $0.04$ without noise to $0.23$, $0.89$ and $1.00$ for $s=10^{-7},10^{-6},10^{-5}$ and continues to increase in $t$, whereas shared noise of the same scales gives $0.04$, $0.05$ and $0.10$. Shared noise also moves the common trajectory and the gate geometry, but it does not act directly on the difference between the units. This is consistent with the escape route through symmetry-breaking noise noted by \citet{joudaki2026}.

\section{Notation}\label{app:notation}

Symbols are listed with the place where they are defined. $C$ and $c$ denote positive constants whose values may change between occurrences; local symbols are marked with the appendices to which they are confined.

\newcommand{\nrow}[3]{#1 & #2 & #3\\}
\newcommand{\ntab}[2]{\par\medskip\noindent\begin{minipage}{\textwidth}\textbf{#1}\par\smallskip\begin{center}\small\begin{tabular}{@{}>{\raggedright\arraybackslash}p{0.21\textwidth}>{\raggedright\arraybackslash}p{0.57\textwidth}>{\raggedright\arraybackslash}p{0.16\textwidth}@{}}\toprule Symbol & Meaning & Where\\\midrule #2\bottomrule\end{tabular}\end{center}\end{minipage}}

\ntab{Model and training}{
\nrow{$X$, $R$, $\hat X$}{input, uniform on $\{\norm x\le2\}$; $X=R\hat X$ with $\E R^2=2$}{Sec.~\ref{sec:setup}}
\nrow{$\theta=(a_1,w_1,a_2,w_2)$}{all six parameters; $f_\theta(x)=\sum_ia_i\relu(w_i^\top x)$, $\relu(z)=\max\{z,0\}$}{Sec.~\ref{sec:setup}}
\nrow{$y_s$, $y_t$}{source label $\relu(x_1)$, target label $\norm x$}{Sec.~\ref{sec:setup}}
\nrow{$R_y$, $F_y$, $\ell_y$}{population risk, objective $R_y+\frac\lambda2\norm\theta^2$, sample loss; subscripts $s,t$ for source and target}{Sec.~\ref{sec:setup}}
\nrow{$\lambda$}{weight decay, any fixed $\lambda\in(0,1/5]$ in the theory (App.~\ref{app:lambda}); $\lambda=1/20$ in numerical values and simulations}{Sec.~\ref{sec:setup}}
\nrow{$b$, $\eta$, $n$, $N$}{batch size, step size, update index, number of updates; physical time is $n\eta$}{Sec.~\ref{sec:setup}}
\nrow{GF, GD}{gradient flow $\dot\theta=-\nabla F_y$; GD = its Euler steps with step $\eta$ (simulations)}{Secs.~\ref{sec:setup}, \ref{sec:experiments}}
\nrow{$T$, $H$}{source training time ($\lfloor T/\eta\rfloor$ updates), target horizon; both physical times}{Sec.~\ref{sec:setup}}
\nrow{$\tau_f$, $\tau_w(T)$, $\tau_{\rm GF}(T)$}{first target time with $R_t\le0.4$: fresh SGD, warm SGD, GF in both phases}{Sec.~\ref{sec:setup}}
}

\ntab{Geometry, initialization and risks}{
\nrow{$r_i$, $\omega_i$, $B$}{$r_i=\norm{w_i}$, $\omega_i=a_ir_i$, total amplitude $B=\omega_1+\omega_2$}{Sec.~\ref{sec:setup}}
\nrow{$\vartheta_i$, $\delta$}{angle of $w_i$ from $e_1$; angle between $w_1$ and $w_2$; the gates disagree with probability $\delta/\pi$}{Secs.~\ref{sec:setup}, \ref{sec:mechanism}}
\nrow{$\calD$, $\calD_0$, $K$}{positive heads, $0<\omega_i<1$, hidden vectors on opposite sides of $e_1$ at angle below $\pi$ (and the neuron exchange); $\calD\cap\{B<3/4\}$; explicit compact subset with $\Pp(K)>0.01064$}{Sec.~\ref{sec:setup}, App.~\ref{app:explicitK}}
\nrow{$q$, $A$, $s_*$}{$q=2/\pi$; $A=q-\lambda$ ($\approx0.587$ at $\lambda=1/20$); $s_*=\sqrt A$}{Secs.~\ref{sec:setup}, \ref{sec:mechanism}}
\nrow{$B_s$, $B_t$}{source and target clone amplitudes, $1-2\lambda$ and $2(q-\lambda)$ ($0.9$ and ${\approx}1.173$ at $\lambda=1/20$)}{Sec.~\ref{sec:mechanism}}
\nrow{$R_{\rm col}$, $R_{\rm good}$}{$1-q^2\approx0.5947$, lower bound for proportional units; $1-2q^2+2\lambda^2$ (${\approx}0.1944$ at $\lambda=1/20$), limit of target GF from $K$}{Sec.~\ref{sec:setup}}
\nrow{$d$, $d_T$, $m_i$}{projective separation $\norm{w_1/a_1-w_2/a_2}$, its value at the source checkpoint; $m_i=r_i/a_i$, with $d^2=(m_1-m_2)^2+2m_1m_2(1-\cos\delta)$ and $d=2\sin(\delta/2)$ when balanced}{Secs.~\ref{sec:intro}, \ref{sec:setup}, \ref{sec:proofmap}}
\nrow{$\budget$, $\budget_T$, $\delta_0$}{disagreement budget $b\delta/\eta$; its value $b\delta_T/\eta$ after source time $T$; angle at the task switch}{Secs.~\ref{sec:intro}, \ref{sec:experiments}}
\nrow{$\delta_s(T)$, $C_s$}{source-GF angle at time $T$ and its coefficient, $\delta_s(T)=C_se^{-\lambda T}(1+o(1))$}{Sec.~\ref{sec:mechanism}, App.~\ref{app:sourceescape}}
\nrow{$\delta_T$, $\theta_s^{\rm GF}(T)$}{angle at a source checkpoint at time $T$; source-GF checkpoint}{Secs.~\ref{sec:checkpoint}, \ref{sec:experiments}, App.~\ref{app:experiments}}
\nrow{$c_T$, $c_H$}{exponents in $T=c_T\log(1/\eta)$ and $H=\eta^{-c_H}$}{Sec.~\ref{sec:checkpoint}}
\nrow{$T_1$, $\eta_1$, $c_p$}{constants of Corollary~\ref{sgd:samecheckpoint}; $c_p=p/(2C)$ in the necessary sizes}{Sec.~\ref{sec:checkpoint}}
}

\ntab{Source SGD (Appendix~\ref{app:sourceSGD})}{
\nrow{$u_i$, $\ell_a$}{$u_i=w_i/a_i$, so $d=\norm{u_1-u_2}$; source allocation $\log(a_1/a_2)$}{Sec.~\ref{sec:proofmap}, App.~\ref{app:sourceSGD}}
\nrow{$\gamma$}{expected contraction rate of $d$, any value below $\lambda$; $\gamma\ge\lambda/2$ is used}{Sec.~\ref{sec:setup}, App.~\ref{app:sourceonestep}}
\nrow{$r_0$, $\mathcal G$}{radius of the source neighborhood in the one-step estimate; event that a sample's gates disagree ($\mathcal G_n$ for a batch in Sec.~\ref{app:hazard})}{App.~\ref{app:sourceonestep}}
\nrow{$\varsigma$, $\Phi_v$, $\mathcal H$, $\mathcal E$}{$\varsigma=1-\eta\lambda$; projective map with gates anchored at unit 2; its first-order term $\mathcal H(v,u)=(v^\top u)I+uv^\top$; error matrix in \eqref{sgd:matchedsource}}{App.~\ref{app:sourceonestep}}
\nrow{$m$, $r$}{allocation margin; radius of the inner normal tube}{App.~\ref{app:sourceconfine}}
\nrow{$t_{\rm blk}$, $\delta_{\rm burn}$, $c_{\rm burn}$}{source block duration; separation reached after burn-in; exponent of the burn-in failure probability}{Apps.~\ref{app:sourceconfine}, \ref{app:entry}}
\nrow{$\varepsilon$, $L$, $E_T$}{allowance for unfavorable allocation drift; burn-in time; good source event}{Thm.~\ref{sgd:main}, Sec.~\ref{sec:proofmap}, App.~\ref{app:entry}}
}

\ntab{Mechanism and proof overview (Sections~\ref{sec:mechanism}--\ref{sec:proofmap})}{
\nrow{$\varphi$, $V$, $V_{\rm com}$, $v$, $v_{\rm com}$}{transverse splitting at the target clone; population field $V=-\nabla F_t$ and common-gate mean field $V_{\rm com}$; half the differences of their second hidden coordinates ($v$ is also the batch vector in App.~\ref{app:sourceonestep}, Sec.~\ref{app:commongate})}{Sec.~\ref{sec:mechanism}}
\nrow{$U$, $u$}{$U=a_1w_1+a_2w_2$; unit hidden direction $u=w/r$}{Secs.~\ref{sec:mechanism}, \ref{app:commongate}}
\nrow{$M(u)$, $P_v$}{$\bigl(\begin{smallmatrix}0&u^\top\\u&0\end{smallmatrix}\bigr)$; orthogonal projection onto $v^\perp$}{Secs.~\ref{app:commongate}, \ref{app:projective}}
\nrow{$e=(e_a,e_w)$, $\mathcal F_n$}{head and hidden parts of $e$; $\sigma$-field of the first $n$ batches}{Secs.~\ref{sec:proofmap}, \ref{app:projective}}
}

\ntab{Target SGD (Section~\ref{app:targetSGD})}{
\nrow{$z_i$, $h$, $w$, $r$, $d_z$}{$z_i=(a_i,w_i)$; $h=(z_1+z_2)/\sqrt2=(h_a,w)$, $r=\norm w$; $d_z=(z_1-z_2)/\sqrt2$ (distinct from $d$). Alone, $r$ is also the tube radius in App.~\ref{app:sourceconfine} and a hidden norm in App.~\ref{app:controls}}{Sec.~\ref{app:targetSGD}}
\nrow{$\rho$, $e$, $\Delta$, $\Delta_0$}{target allocation $h^\top d_z/\norm h^2$ ($=\tanh(\ell_a/2)$ on clones; thresholds $\rho_0<\rho_1<\rho_2$ in Proposition~\ref{sgd:targettheorem}); $e=d_z-\rho h$; $\Delta=\norm e$; its initial value}{Sec.~\ref{app:targetSGD}, App.~\ref{app:composition}; $\Delta$ also Sec.~\ref{sec:proofmap}}
\nrow{$S$, $k$, $B$}{$S=1+\rho^2$, $k=h_a^2-r^2$, $B=Sh_ar$ (equal to $\omega_1+\omega_2$ on the proportional family)}{Sec.~\ref{app:targetSGD}}
\nrow{$G$, $Q$, $\tilde\theta_n$}{common-gate random matrix, $Q=I+\eta G$, auxiliary process $z_i^+=Qz_i$ (frozen after the stop)}{Secs.~\ref{app:commongate}, \ref{app:hazard}}
\nrow{$\nu$, $\epsilon_0$}{normal and projective entry tolerances}{Proposition~\ref{sgd:targettheorem}}
\nrow{$\kappa$}{decay rate of the Lyapunov functions $W_s$, $W_t$ along the population flow}{Apps.~\ref{app:sourceconfine}, \ref{app:bridge}}
\nrow{$\zeta_t$, $n_{\rm dis}$, $\epsilon$}{tube exit time; first batch containing a disagreement; projective radius of the tube (distinct from $\varepsilon$)}{Secs.~\ref{app:projective}, \ref{app:hazard}}
}

\ntab{Constants and population analysis}{
\nrow{$c_*$, $\eta_0$, $t_0$}{confinement rate (at most the burn-in, source, target and fresh-tracking rates), step-size cutoff, fresh hitting-time bound in Thm.~\ref{sgd:main}}{Thm.~\ref{sgd:main}, App.~\ref{app:composition}}
\nrow{$\bar\eta$, $\eta_{\rm t}$, $\ell_{\rm blk}$}{step-size thresholds of Lemmas~\ref{sgd:tracking}--\ref{sgd:block} and of Proposition~\ref{sgd:targettheorem}; block duration in Lemma~\ref{sgd:block}}{App.~\ref{app:tracking}, Sec.~\ref{app:targetSGD}}
\nrow{$C_0$, $c_0$, $c_1$, $t_1$, $\eta_2$, $\varrho$}{constants of Corollaries~\ref{sgd:window} and~\ref{sgd:fixedT} ($C_0$ is reused in App.~\ref{app:sourceescape}); $\varrho$ is the level in Lemma~\ref{sgd:block}}{Sec.~\ref{sec:setup}, App.~\ref{app:tracking}}
\nrow{$J$, $g$}{ReLU kernel $J(t)=(\sin t+(\pi-t)\cos t)/(2\pi)$ and $g=-J'$}{App.~\ref{app:populationequations}; $g$ also Sec.~\ref{sec:mechanism}}
\nrow{$W_s$, $W_t$, $R_W$, $\ell_0$}{Lyapunov functions for the source and target normal coordinates; target tube level; duration of the target amplitude bridge}{Apps.~\ref{app:sourceconfine}, \ref{app:bridge}}
\nrow{$\mathcal A\Subset\mathcal A'$, $\zeta_s$, $\zeta'$}{source allocation intervals; source stopping time; exit time from the constraint set in Lemma~\ref{sgd:block}}{Apps.~\ref{app:sourceconfine}, \ref{app:tracking}}
\nrow{$\alpha$, $\beta$, $q_i$, $h_i$, $k_i$, $s_i$}{$\alpha=\vartheta_1$, $\beta=-\vartheta_2$; local to Apps.~\ref{app:populationequations}--\ref{app:sourceescape}: $q_i$ the factor in $\dot a_i=r_iq_i-\lambda a_i$, $h_i=a_i/r_i$, $k_i=a_i^2-r_i^2$, $s_i=a_i+r_i$}{App.~\ref{app:populationequations}}
\nrow{$\delta_\star$, $D(\delta)$, $m$}{fixed small target angle for the escape time; $D(\delta)=1/2-J(\delta)$; lower amplitude bound (unrelated to $m_i$ in Sec.~\ref{sec:proofmap})}{App.~\ref{app:sourceescape}}
\nrow{$m$, $M$}{margins $10^{-4}$ and $5$ in the definition of $K$}{App.~\ref{app:explicitK}}
}

\ntab{Experiments}{
\nrow{$\kappa_{\rm eff}$}{effective-intensity ratio $-\pi\lambda\log(1-\hat p)/\budget$}{App.~\ref{app:experiments}}
\nrow{$T_{1/2}(\eta)$}{source time at which half of the warm runs recover}{App.~\ref{app:joint}}
\nrow{$s$, $\Lambda(t)$}{parameter-noise scale; cumulative Poisson intensity of the reference}{App.~\ref{app:controls}}
}


{\small
\begin{thebibliography}{99}
\bibitem[Ash and Adams(2020)]{ash2020} Jordan T. Ash and Ryan P. Adams.
\newblock \href{https://papers.nips.cc/paper_files/paper/2020/file/288cd2567953f06e460a33951f55daaf-Paper.pdf}{On warm-starting neural network training}.
\newblock In \emph{NeurIPS}, 2020.

\bibitem[Ben~Arous et~al.(2021)]{benarous2021} G\'erard Ben Arous, Reza Gheissari, and Aukosh Jagannath.
\newblock \href{https://jmlr.org/papers/v22/20-1288.html}{Online stochastic gradient descent on non-convex losses from high-dimensional inference}.
\newblock \emph{Journal of Machine Learning Research}, 22(106):1--51, 2021.

\bibitem[Ben~Arous et~al.(2022)]{benarous2022} G\'erard Ben Arous, Reza Gheissari, and Aukosh Jagannath.
\newblock \href{https://proceedings.neurips.cc/paper_files/paper/2022/file/a224ff18cc99a71751aa2b79118604da-Paper-Conference.pdf}{High-dimensional limit theorems for {SGD}: Effective dynamics and critical scaling}.
\newblock In \emph{NeurIPS}, 2022.

\bibitem[Bena\"{\i}m(1999)]{benaim1999} Michel Bena\"{\i}m.
\newblock \href{https://numdam.org/item/SPS_1999__33__1_0/}{Dynamics of stochastic approximation algorithms}.
\newblock In \emph{S\'eminaire de Probabilit\'es XXXIII}, Lecture Notes in Mathematics 1709, pages 1--68. Springer, 1999.

\bibitem[Bianchi et~al.(2024)]{bianchi2021} Pascal Bianchi, Walid Hachem, and Sholom Schechtman.
\newblock \href{https://doi.org/10.1287/moor.2021.0194}{Stochastic subgradient descent escapes active strict saddles on weakly convex functions}.
\newblock \emph{Mathematics of Operations Research}, 49(3):1761--1790, 2024.

\bibitem[Brandi\`ere and Duflo(1996)]{brandiere1996} Odile Brandi\`ere and Marie Duflo.
\newblock \href{https://www.numdam.org/item/AIHPB_1996__32_3_395_0/}{Les algorithmes stochastiques contournent-ils les pi\`eges}?
\newblock \emph{Annales de l'IHP Probabilit\'es et Statistiques}, 32(3):395--427, 1996.

\bibitem[Chen et~al.(2023)]{chen2023} Feng Chen, Daniel Kunin, Atsushi Yamamura, and Surya Ganguli.
\newblock \href{https://proceedings.neurips.cc/paper_files/paper/2023/hash/6e4432b912599d11609b9cdf98c823c5-Abstract.html}{Stochastic collapse: How gradient noise attracts {SGD} dynamics towards simpler subnetworks}.
\newblock In \emph{NeurIPS}, 2023.

\bibitem[Davis et~al.(2026)]{davis2021} Damek Davis, Dmitriy Drusvyatskiy, and Liwei Jiang.
\newblock \href{https://doi.org/10.1007/s10208-025-09691-0}{Active manifolds, stratifications, and convergence to local minima in nonsmooth optimization}.
\newblock \emph{Foundations of Computational Mathematics}, 26(2):779--861, 2026.

\bibitem[Dohare et~al.(2024)]{dohare2024} Shibhansh Dohare, J. Fernando Hernandez-Garcia, Qingfeng Lan, Parash Rahman, A. Rupam Mahmood, and Richard S. Sutton.
\newblock \href{https://www.nature.com/articles/s41586-024-07711-7}{Loss of plasticity in deep continual learning}.
\newblock \emph{Nature}, 632:768--774, 2024.

\bibitem[Du et~al.(2018)]{du2018} Simon S. Du, Wei Hu, and Jason D. Lee.
\newblock \href{https://proceedings.neurips.cc/paper/2018/hash/fe131d7f5a6b38b23cc967316c13dae2-Abstract.html}{Algorithmic regularization in learning deep homogeneous models: Layers are automatically balanced}.
\newblock In \emph{NeurIPS}, 2018.

\bibitem[Freidlin and Wentzell(2012)]{freidlin2012} Mark I. Freidlin and Alexander D. Wentzell.
\newblock \href{https://link.springer.com/book/10.1007/978-3-642-25847-3}{\emph{Random Perturbations of Dynamical Systems}}.
\newblock Springer, 3rd edition, 2012.

\bibitem[Fukumizu and Amari(2000)]{fukumizu2000} Kenji Fukumizu and Shun-ichi Amari.
\newblock \href{https://doi.org/10.1016/S0893-6080(00)00009-5}{Local minima and plateaus in hierarchical structures of multilayer perceptrons}.
\newblock \emph{Neural Networks}, 13(3):317--327, 2000.

\bibitem[Ge et~al.(2015)]{ge2015} Rong Ge, Furong Huang, Chi Jin, and Yang Yuan.
\newblock \href{https://proceedings.mlr.press/v40/Ge15.pdf}{Escaping from saddle points---online stochastic gradient for tensor decomposition}.
\newblock In \emph{COLT}, 2015.

\bibitem[Jin et~al.(2017)]{jin2017} Chi Jin, Rong Ge, Praneeth Netrapalli, Sham M. Kakade, and Michael I. Jordan.
\newblock \href{https://proceedings.mlr.press/v70/jin17a.html}{How to escape saddle points efficiently}.
\newblock In \emph{ICML}, 2017.

\bibitem[Jin et~al.(2018)]{jin2018} Chi Jin, Lydia T. Liu, Rong Ge, and Michael I. Jordan.
\newblock \href{https://proceedings.neurips.cc/paper/2018/hash/da4902cb0bc38210839714ebdcf0efc3-Abstract.html}{On the local minima of the empirical risk}.
\newblock In \emph{NeurIPS}, 2018.

\bibitem[Joudaki et~al.(2026)]{joudaki2026} Amir Joudaki, Giulia Lanzillotta, Mohammad Samragh, Iman Mirzadeh, Keivan Alizadeh-Vahid, Thomas Hofmann, Mehrdad Farajtabar, and Fartash Faghri.
\newblock \href{https://openreview.net/forum?id=g6kof5fSba}{Barriers for learning in an evolving world: Mathematical understanding of loss of plasticity}.
\newblock In \emph{ICLR}, 2026.

\bibitem[Kumar et~al.(2021)]{kumar2021} Aviral Kumar, Rishabh Agarwal, Dibya Ghosh, and Sergey Levine.
\newblock \href{https://arxiv.org/abs/2010.14498}{Implicit under-parameterization inhibits data-efficient deep reinforcement learning}.
\newblock In \emph{ICLR}, 2021.

\bibitem[Kumar et~al.(2025)]{kumar2024} Saurabh Kumar, Henrik Marklund, and Benjamin Van Roy.
\newblock \href{https://proceedings.mlr.press/v274/kumar25a.html}{Maintaining plasticity in continual learning via regenerative regularization}.
\newblock In \emph{CoLLAs 2024}, PMLR 274:410--430, 2025.

\bibitem[Kushner and Yin(2003)]{kushner2003} Harold J. Kushner and G. George Yin.
\newblock \href{https://link.springer.com/book/10.1007/b97441}{\emph{Stochastic Approximation and Recursive Algorithms and Applications}}.
\newblock Springer, 2nd edition, 2003.

\bibitem[Lewandowski et~al.(2023)]{lewandowski2023} Alex Lewandowski, Haruto Tanaka, Dale Schuurmans, and Marlos C. Machado.
\newblock \href{https://arxiv.org/abs/2312.00246}{Directions of curvature as an explanation for loss of plasticity}.
\newblock arXiv:2312.00246, 2023.

\bibitem[Li et~al.(2026)]{li2026two} Binghua Li, Mengzhe Li, Denny Wu, and Tianhao Wang.
\newblock \href{https://openreview.net/forum?id=4FX5HdRtjL}{Gradient descent on two {ReLU} neurons: Global landscape and bifurcation dynamics}.
\newblock In \emph{ICML Workshop on High Dimensional Learning Dynamics: The Science of Scaling}, 2026.

\bibitem[Li et~al.(2017)]{li2017sme} Qianxiao Li, Cheng Tai, and Weinan E.
\newblock \href{https://proceedings.mlr.press/v70/li17f/li17f.pdf}{Stochastic modified equations and adaptive stochastic gradient algorithms}.
\newblock In \emph{ICML}, 2017.

\bibitem[Lyle et~al.(2023)]{lyle2023} Clare Lyle, Zeyu Zheng, Evgenii Nikishin, Bernardo Avila Pires, Razvan Pascanu, and Will Dabney.
\newblock \href{https://proceedings.mlr.press/v202/lyle23b.html}{Understanding plasticity in neural networks}.
\newblock In \emph{ICML}, 2023.

\bibitem[Lyle et~al.(2025)]{lyle2024} Clare Lyle, Zeyu Zheng, Khimya Khetarpal, Hado van Hasselt, Razvan Pascanu, James Martens, and Will Dabney.
\newblock \href{https://proceedings.mlr.press/v274/lyle25a.html}{Disentangling the causes of plasticity loss in neural networks}.
\newblock In \emph{CoLLAs 2024}, PMLR 274:750--783, 2025.

\bibitem[Maennel et~al.(2018)]{maennel2018} Hartmut Maennel, Olivier Bousquet, and Sylvain Gelly.
\newblock \href{https://arxiv.org/abs/1803.08367}{Gradient descent quantizes {ReLU} network features}.
\newblock arXiv:1803.08367, 2018.

\bibitem[Mertikopoulos et~al.(2020)]{mertikopoulos2020} Panayotis Mertikopoulos, Nadav Hallak, Ali Kavis, and Volkan Cevher.
\newblock \href{https://papers.nips.cc/paper/2020/hash/0cb5ebb1b34ec343dfe135db691e4a85-Abstract.html}{On the almost sure convergence of stochastic gradient descent in non-convex problems}.
\newblock In \emph{NeurIPS}, 2020.

\bibitem[Nikishin et~al.(2022)]{nikishin2022} Evgenii Nikishin, Max Schwarzer, Pierluca D'Oro, Pierre-Luc Bacon, and Aaron Courville.
\newblock \href{https://proceedings.mlr.press/v162/nikishin22a.html}{The primacy bias in deep reinforcement learning}.
\newblock In \emph{ICML}, 2022.

\bibitem[Pemantle(1990)]{pemantle1990} Robin Pemantle.
\newblock \href{https://doi.org/10.1214/aop/1176990853}{Nonconvergence to unstable points in urn models and stochastic approximations}.
\newblock \emph{The Annals of Probability}, 18(2):698--712, 1990.

\bibitem[Saad and Solla(1995)]{saad1995} David Saad and Sara A. Solla.
\newblock \href{https://doi.org/10.1103/PhysRevE.52.4225}{On-line learning in soft committee machines}.
\newblock \emph{Physical Review E}, 52(4):4225--4243, 1995.

\bibitem[{\c{S}}im{\c{s}}ek et~al.(2021)]{simsek2021} Berfin \c{S}im\c{s}ek, Fran\c{c}ois Ged, Arthur Jacot, Francesco Spadaro, Cl\'ement Hongler, Wulfram Gerstner, and Johanni Brea.
\newblock \href{https://proceedings.mlr.press/v139/simsek21a.html}{Geometry of the loss landscape in overparameterized neural networks: Symmetries and invariances}.
\newblock In \emph{ICML}, 2021.

\bibitem[Sokar et~al.(2023)]{sokar2023} Ghada Sokar, Rishabh Agarwal, Pablo Samuel Castro, and Utku Evci.
\newblock \href{https://proceedings.mlr.press/v202/sokar23a.html}{The dormant neuron phenomenon in deep reinforcement learning}.
\newblock In \emph{ICML}, 2023.

\bibitem[Tian(2017)]{tian2017} Yuandong Tian.
\newblock \href{https://proceedings.mlr.press/v70/tian17a.html}{An analytical formula of population gradient for two-layered {ReLU} network and its applications in convergence and critical point analysis}.
\newblock In \emph{ICML}, 2017.

\bibitem[Wang et~al.(2026)]{wang2026} Jiuqi Wang, Jayanth Srinivasa, Claire Chen, Shuze Daniel Liu, Ali Payani, and Shangtong Zhang.
\newblock \href{https://arxiv.org/abs/2605.09044}{Predicting plasticity in deep continual learning: A theoretical perspective}.
\newblock arXiv:2605.09044, 2026.

\bibitem[Wei et~al.(2008)]{wei2008} Haikun Wei, Jun Zhang, Florent Cousseau, Tomoko Ozeki, and Shun-ichi Amari.
\newblock \href{https://doi.org/10.1162/neco.2007.12-06-414}{Dynamics of learning near singularities in layered networks}.
\newblock \emph{Neural Computation}, 20(3):813--843, 2008.

\bibitem[Xu and Du(2023)]{xu2023} Weihang Xu and Simon S. Du.
\newblock \href{https://proceedings.mlr.press/v195/xu23a.html}{Over-parameterization exponentially slows down gradient descent for learning a single neuron}.
\newblock In \emph{COLT}, PMLR 195, pages 1155--1198, 2023.

\bibitem[Yehudai and Shamir(2020)]{yehudai2020} Gilad Yehudai and Ohad Shamir.
\newblock \href{https://proceedings.mlr.press/v125/yehudai20a.html}{Learning a single neuron with gradient methods}.
\newblock In \emph{COLT}, PMLR 125, 2020.

\bibitem[Zhong et~al.(2017)]{zhong2017} Kai Zhong, Zhao Song, Prateek Jain, Peter L. Bartlett, and Inderjit S. Dhillon.
\newblock \href{https://proceedings.mlr.press/v70/zhong17a.html}{Recovery guarantees for one-hidden-layer neural networks}.
\newblock In \emph{ICML}, 2017.

\bibitem[Ziyin(2024)]{ziyin2024} Liu Ziyin.
\newblock \href{https://proceedings.mlr.press/v235/ziyin24a.html}{Symmetry induces structure and constraint of learning}.
\newblock In \emph{ICML}, 2024.

\bibitem[Ziyin et~al.(2022)]{ziyin2022} Liu Ziyin, Botao Li, James B. Simon, and Masahito Ueda.
\newblock \href{https://openreview.net/forum?id=9XhPLAjjRB}{{SGD} can converge to local maxima}.
\newblock In \emph{ICLR}, 2022.

\bibitem[Ziyin et~al.(2023)]{typeii} Liu Ziyin, Botao Li, Tomer Galanti, and Masahito Ueda.
\newblock \href{https://arxiv.org/abs/2303.13093}{Type-{II} saddles and probabilistic stability of stochastic gradient descent}.
\newblock arXiv:2303.13093, 2023.

\bibitem[Ziyin et~al.(2024)]{ziyin2024noise} Liu Ziyin, Mingze Wang, Hongchao Li, and Lei Wu.
\newblock \href{https://proceedings.neurips.cc/paper_files/paper/2024/hash/aa933b5abc1be30baece1d230ec575a7-Abstract-Conference.html}{Parameter symmetry and noise equilibrium of stochastic gradient descent}.
\newblock In \emph{NeurIPS}, 2024.

\end{thebibliography}
}
\end{document}